\documentclass[mnsc,nonblindrev]{informs3}

\usepackage{etoolbox}

\AtBeginEnvironment{table}{\centering}
\AtBeginEnvironment{figure}{\centering}

\OneAndAHalfSpacedXI

\usepackage{natbib}
 \bibpunct[, ]{(}{)}{,}{a}{}{,}%
 \def\bibfont{\small}%

\usepackage[pdfencoding=auto]{hyperref}
\usepackage{bookmark}

\usepackage{multirow}
\usepackage{array}
\usepackage{booktabs}
\usepackage{amsmath}

\usepackage{thmtools, thm-restate}
 \usepackage[ruled]{algorithm2e}
\usepackage{algorithmic}
\usepackage{enumitem}
\usepackage{bbm}

\usepackage{comment}
\usepackage{setspace}
\usepackage{tikz}
\usepackage{pgfplots}
\usepgfplotslibrary{patchplots}
\usetikzlibrary{patterns, positioning, arrows}
\pgfplotsset{compat=1.15}
\usepackage{tikz-3dplot}
\usetikzlibrary{3d, calc}

\pgfmathdeclarefunction{noisydecay}{1}{%
    \pgfmathparse{(1/sqrt(#1)) + (0.1 * (rand - 0.5))}%
}
\TheoremsNumberedThrough  

\EquationsNumberedThrough

\numberwithin{equation}{section}
\numberwithin{remark}{section}
\numberwithin{lemma}{section}
\numberwithin{theorem}{section}
\numberwithin{corollary}{section}
\numberwithin{assumption}{section}
\numberwithin{proposition}{section}

\makeatletter
\long\def\@maketablecaption#1#2{
  \EGT\TableCaptionFontStyle
  \noindent
    {\TableNameFontStyle #1.}\enskip #2\HD{0}{6}\endgraf%
}

\long\def\@makefigurecaption#1#2{
  \EGT\FigureCaptionFontStyle
  \noindent
    {\FigureNameFontStyle #1.}\enskip #2\HD{0}{9}\endgraf
}
\makeatother

\usepackage{amssymb}
\usepackage{mathtools}
\usepackage{bm}
\usepackage{microtype}
\usetikzlibrary{arrows.meta, decorations.pathreplacing}

\newcommand{\Cov}{\mathrm{Cov}}
\newcommand{\E}{\mathbb{E}}
\newcommand{\R}{\mathbb{R}}
\newcommand{\Rplus}{\mathbb{R}_{+}}
\newcommand{\KL}{\mathrm{KL}}
\newcommand{\Var}{\mathrm{Var}}
\newcommand{\vct}[1]{\bm{#1}}
\newcommand{\one}{\bm{1}}
\newcommand{\Tr}{\mathrm{Tr}}

\newcommand{\Rpos}{\mathbb{R}_{>0}}
\newcommand{\Reg}{\mathrm{Regret}}
\newcommand{\hts}{\mathrm{Het}}
\begin{document}

\RUNAUTHOR{Aznag, Cummings, and Elmachtoub}
\RUNTITLE{A Complexity Measure for Active Learning in Multi-group Mean Estimation}
\TITLE{A Complexity Measure for Active Learning in Multi-group Mean Estimation}
\ARTICLEAUTHORS{%
\AUTHOR{Abdellah Aznag\qquad\qquad Rachel Cummings \qquad\qquad Adam N. Elmachtoub}
\AFF{Department of Industrial Engineering and Operations Research \& Data Science Institute, Columbia University\\ Emails: \texttt{\{aa4693, rc2239, ae2516\}@columbia.edu}}
}
\ABSTRACT{We study a \emph{max-risk} objective for active learning in a multi-group mean estimation $d$-armed bandits: a learner adaptively allocates a budget of $T$ samples across $d$ groups to minimize the worst-case uncertainty index $\max_{k\in[d]}\sigma_k^2/n_k$, where $\sigma_k$ is the standard deviation of the distribution of arm $d$, and $n_k$ is the number of times arm $d$ is sampled.  We develop a local minimax framework and prove the first general lower bound for this objective, valid for any finite-variance hypothesis class. The bound separates difficulty into three orthogonal factors: a \emph{budget} term, a \emph{heteroscedasticity} index measuring how unevenly the uncertainty is spread across arms, and a model-dependent complexity measure, the \emph{Variance Local Curvature} ($\mathrm{VLC}$), which captures how much information a local change of variance creates inside the hypothesis class. 
For smooth classes, the $\mathrm{VLC}$ is a reparametrization of a variance--Fisher information, with closed-form values for common families. Benchmarking against the strongest available upper bound \citep{aznag2025activelearningframeworkmultigroup} shows near-optimality up to logarithmic factors in broad regimes, and pinpoints a systematic gap in highly heterogeneous instances. Our proof introduces two key ingredients: a loss-induced $\ell_1$ geometry on the decision space, and a representation-based instance generator that reduces hard-instance construction to an explicit random matrix calculation.}

\KEYWORDS{Active learning, multi-armed bandits, information-theoretic lower bounds, non-additive regrets.}

\maketitle

\section{Introduction}

Active learning studies sequential sampling rules that choose where to collect the next observation under budget constraints. This differs from classical estimation, where data are often treated as given. In this paper, we are concerned with settings where the key decision is where to sample from next. This setting appears in online experimentation, clinical trials, and any application where we estimate group-level quantities under a shared budget and want a controlled, balanced estimation noise. While a substantial literature designs high-performing adaptive policies for these problems, far less is known about the fundamental performance limits they face.

\paragraph{Model.} We consider a multi-armed bandit setting where the goal is to actively learn the means of the arms' distributions, and where accuracy is measured by the uncertainty index per-arm $\sigma^2/n$. Fix a hypothesis class $\mathcal H$ of distributions on $\R$ with finite variance.
An \emph{instance} is a $d$-tuple of distributions $\vct{D}=(D_1,\dots,D_d)\in\mathcal H^d$, where each $D_k$
has mean $\mu_k$ and standard deviation $\sigma(D_k)>0$. We write
$\vct{\sigma}(\vct{D}):=(\sigma_1(\vct{D}),\dots,\sigma_d(\vct{D})) \in\R_+^d$, where $\sigma_k(\vct{D}) := \sigma(D_k)$.

The class $\mathcal H$ is known, but the instance $\vct{D}$ (and hence $\vct{\sigma}(\vct{D})$) is unknown. For a fixed, known horizon $T$, a (possibly adaptive) \emph{policy} $\pi$ sequentially chooses arms $A_1,\dots,A_T\in[d]$ and observes the
corresponding samples. Let
\[
n_k \;:=\; \sum_{t=1}^T \mathbf 1\{A_t=k\}
\qquad (k\in[d])
\]
be the terminal number of pulls of arm $k$, and write $\vct{n}:=(n_1,\dots,n_d)$.
The terminal risk is
\[
R(\vct{n};\vct{\sigma})\;:=\;\max_{k\in[d]}\ \frac{\sigma_k^2}{n_k},
\]
and the corresponding optimal benchmark risk is
\begin{equation}\label{eq:oracle-risk}
R^\star(\vct{\sigma})\;:=\;\min_{\substack{\vct{n}\in\R_+^d:\\ \sum_k n_k=T}} R(\vct{n};\vct{\sigma})
\;=\;\frac{\|\vct{\sigma}\|_2^2}{T}.
\end{equation}
\begin{remark}[Confidence-interval representation]
Each per-arm index $\sigma_k^2/n_k$ is (up to a constant) the squared width of a confidence interval for the mean
$\mu_k$ built from the $n_k$ samples of arm $k$. The risk $R(\vct n;\vct\sigma)=\max_{k\in[d]}\sigma_k^2/n_k$ is
therefore the \emph{widest} such interval across arms, and minimizing it amounts to controlling the worst-estimated
mean rather than the best one.
\end{remark}

We measure performance by \textit{normalized regret}, defined for a policy $\pi$ and instance $\vct{D}\in\mathcal H^d$ as
\begin{equation}\label{def:regret}
    \Reg(\pi,\vct{D})
\;:=\;
\E_{\vct{n}\sim(\pi,\vct{D})}\!\left[\frac{R(\vct{n};\vct{\sigma}(\vct{D}))}{R^\star(\vct{\sigma}(\vct{D}))}-1\right],
\end{equation}
where the randomness is induced by the policy $\pi$ and the samples drawn from $\vct{D}$ through the random counts $\vct{n}$. We study information-theoretic lower bounds on the normalized regret defined in \eqref{def:regret}.
\begin{remark}[Effect of normalizing]
We use normalized rather than standard regret because it is scale-free and captures the right asymptotics: any
policy pulling every arm $\sim T$ times already drives the \emph{un}normalized excess risk to zero, so standard
regret cannot isolate the difficulty of learning the variances. Dividing by $R^\star(\vct\sigma)$ makes the
quantity vanish only as the optimal allocation is approached. One recovers a bound on standard regret simply by
multiplying the bound of Theorem~\ref{thm:main} through by $R^\star(\vct\sigma)$.
\end{remark}
\paragraph{Local neighborhood.}
Fix a base standard-deviation vector $\vct{\sigma}\in\R_+^d$ and define its normalized variance profile
\[
\mathcal N(\vct{\sigma})\;:=\;\frac{\vct{\sigma}^{\odot 2}}{\|\vct{\sigma}\|_2^2}\in\Delta_d.
\]
$\mathcal{N}(\vct{\sigma})$ is scale-free and encodes the optimal normalized allocation that yields $R^*(\vct{\sigma})$. For $\rho\in(0,1)$, define the local instance class
\[
\mathcal H_{\rho}(\vct{\sigma})
\;:=\;
\Big\{\vct{D}\in\mathcal H^d:\ \big\|\mathcal N(\vct{\sigma}(\vct{D}))/\mathcal N(\vct{\sigma})-\one\big\|_2\le \rho\Big\}.
\]

\paragraph{Local minimax lower bound.}
Since the arms are unlabeled, we restrict attention to permutation-invariant policies (i.e., policies invariant to permutations of arm labels). This is possible because any policy can be symmetrized by a random permutation of labels at time $0$ without worsening worst-case regret over a permutation-closed instance class; let $\Pi^{\mathrm{sym}}$ denote this class. Fix a radius $\rho\in(0,1)$ and a tolerance $\tau>0$. Define the $\tau$-tolerant symmetric policies by
\[
\Pi^{\mathrm{sym}}_{\rho,\tau}(\vct{\sigma})
\;:=\;
\Big\{\pi\in\Pi^{\mathrm{sym}}:\ \sup_{\vct{D}\in\mathcal H_{\rho}(\vct{\sigma})} \Reg(\pi,\vct{D}) \le \tau\Big\}.
\]

The tolerance $\tau$ restricts attention to policies whose worst-case normalized regret over the local class is
controlled, and is primarily a \emph{proof-technical} device: our lower-bound argument requires prior
control on how a policy allocates its budget on $\mathcal H_{\rho}(\vct{\sigma})$, and without such control the
information-theoretic reduction does not yield a non-vacuous bound. We make the precise role of $\tau$ explicit in
Remark~\ref{rem:tolerance} below. In the learnable asymptotic setting, one is naturally interested in regimes where
$\tau\downarrow 0$. The corresponding local minimax value is defined by
\[
\mathcal V_{\rho,\tau}(\vct{\sigma})
\;:=\;
\inf_{\pi\in\Pi^{\mathrm{sym}}_{\rho,\tau}(\vct{\sigma})}\ \sup_{\vct{D}\in\mathcal H_{\rho}(\vct{\sigma})}\Reg(\pi,\vct{D}),
\]
with the convention $\mathcal{V}_{\rho, \tau}(\boldsymbol{\sigma}) = +\infty$ if $\Pi^{\mathrm{sym}}_{\rho, \tau} = \emptyset$.

\begin{remark}[On the role of the tolerance $\tau$]\label{rem:tolerance}
The tolerance does not change the local minimax value in the learnable regime. Writing
$\tau^\star:=\inf_{\pi\in\Pi^{\mathrm{sym}}}\sup_{\vct D\in\mathcal H_\rho(\vct\sigma)}\Reg(\pi,\vct D)$ for the
unconstrained value, one indeed has $\mathcal V_{\rho,\tau}(\vct\sigma)=\tau^\star$ for every $\tau\ge\tau^\star$
and $\mathcal V_{\rho,\tau}(\vct\sigma)=+\infty$ for $\tau<\tau^\star$. The parameter $\tau$ instead enters through
\emph{analysis}: the lower bound we prove (Theorem~\ref{thm:main}) holds for $\tau$-tolerant policies and carries a
multiplicative penalty $1-\epsilon(\rho,\tau)$. Restricting to small $\tau$ (the learnable regime) sends
$\epsilon(\rho,\tau)\to 0$ and makes the bound sharp, whereas imposing no restriction at all ($\tau=+\infty$) only
yields the weaker penalty $\epsilon(\rho,+\infty)$, which can be vacuous. Since $\mathcal V_{\rho,\tau}=\tau^\star$
throughout the learnable regime, the resulting bound also lower bounds the unconstrained local minimax value.
\end{remark}
\paragraph{Benchmark and notion of sharpness.}
Our goal is to lower bound $\mathcal V_{\rho,\tau}(\vct{\sigma})$. To interpret such a lower bound, suppose that (by separate algorithmic work) one constructs a permutation-invariant policy $\pi^\star$ with a
uniform local worst-case guarantee
\[
\sup_{\vct D\in\mathcal H_\rho(\vct{\sigma})}\Reg(\pi^\star,\vct D) \leq \bar{\mathcal V}_\rho(\vct{\sigma}),
\]
where $\bar{\mathcal V}_\rho(\vct{\sigma})$ is an explicit function of $(T,d,\vct{\sigma}, \rho)$. Then, we have by construction,
\[
\bar{\mathcal V}_\rho(\vct{\sigma})
\;\ge\;
\sup_{\vct D\in\mathcal H_\rho(\vct{\sigma})}\Reg(\pi^{\star},\vct D)
\;\ge\;
\mathcal V_{\rho,\bar{\mathcal V}_\rho(\vct{\sigma})}(\vct{\sigma}).
\]

A meaningful lower bound therefore does two things. First, it must apply at a tolerance level $\bar{\mathcal V}_\rho(\vct{\sigma})$ for which an
admissible policy exists, i.e., $\Pi^{\mathrm{sym}}_{\rho,\bar{\mathcal V}_\rho(\vct{\sigma})}(\vct{\sigma})\neq\emptyset$. Second, it should satisfy a small gap between the achievable guarantee and
the minimax value $\frac{\bar{\mathcal V}_\rho(\vct{\sigma})}{\mathcal V_{\rho,\bar{\mathcal V}_\rho(\vct{\sigma})}(\vct{\sigma})}$. In particular, a $\tilde{O}(1)$ guarantee on this ratio implies local minimax near-optimality.

\subsection{Our Contributions}\label{sec:contrib}

Conceptually, our central novelty is an \emph{information-theoretic} approach of the non-additive, max-type
objective defined in \eqref{def:regret}. Specifically, we show that its local minimax difficulty is governed by a single curvature
functional, the Variance Local Curvature (VLC), rather than by tail conditions on $\mathcal H$. Our main contributions are threefold:

\begin{enumerate}[leftmargin=*, itemsep=0.35em]
    \item \textit{A general local minimax lower bound framework.}
    We introduce a local minimax formulation parameterized by a locality radius $\rho$ and a tolerance $\tau$,
    and prove a lower bound on $\mathcal V_{\rho,\tau}(\vct{\sigma})$ that cleanly separates budget, heteroscedasticity,
    and model curvature. The bound is expressed through a one-arm curvature functional $\mathrm{VLC}_\rho(\cdot\mid\mathcal H)$,
    and applies to general hypothesis classes $\mathcal H$ with finite variance (see Section \ref{sec:vlc}).

\item \textit{Sharp characterization for smooth classes and a principled benchmark.}
For smooth classes, we relate $\mathrm{VLC}_\rho$ to a variance--Fisher information quantity via a quadratic KL expansion,
yielding explicit expressions for many standard families. We then benchmark our lower bound against the strongest available
general upper bound, showing that it is close in broad regimes (up to logarithmic factors) and
pinpointing a systematic gap in highly heterogeneous instances. This comparison provides both a refined understanding of when
current algorithms are provably near-optimal and concrete guidance on how future upper-bound analyses might be sharpened
(see Section~\ref{sec:smooth}).

    \item \textit{A principled method for generating hard instances.}
    We develop an \emph{instance generation} viewpoint for lower bounds: rather than guessing an adversarial family
    directly in $\mathcal H^d$, we parameterize local hard families through a representation map. Concretely, we encode
    nearby decisions by a hypercube code and a linear map $A$, and view the adversary as choosing the full representation
    $x\mapsto Ax$. The construction
    reduces the core lower bound problem to an explicit matrix optimization problem, which can be handled with sharp random matrix theory machinery. To the best of our knowledge, this ``instance generator'' viewpoint is new in this setting and may be
    useful on its own for other non-additive objectives (see Section \ref{sec:proof}).
\end{enumerate}
\subsection{Related Work}

\paragraph{}Our setting is closely related to stochastic multi-armed bandits (MAB) \citep{lattimore} in that data are acquired by pulling arms.
The lower-bound literature for classical MAB is extensive \citep{magureanu2014lipschitz,he2022reduction,combes2014unimodal,cai2021lower},
and many proofs ultimately reduce to controlling the probability of confusing nearby instances via information-theoretic inequalities
(e.g., Pinsker-type arguments \citep{Yu1997}). A key difference is that these techniques are typically tailored to bandit objectives whose
regret is additive over time (cumulative regret) or has a canonical discrete structure (best-arm identification). Our objective is neither:
the normalized regret in \eqref{def:regret} is a non-additive, max-type functional of the terminal allocation. As a result, standard MAB
lower-bound templates do not apply directly and must be adapted to the loss-induced geometry of our decision problem.

\paragraph{}Early work introduced this problem for bounded feedback models, motivated in part by fairness constraints
\citep{sarkies2015data,NBERw26296,de2008choosing}. \cite{activelearning} studied the case where $\mathcal H$
consists of bounded distributions with known support and provided an algorithm with regret
$\tilde O\!\left(\frac{d^3\|\vct\sigma\|_2}{\sigma_{\min}\sqrt{T}}\right)$, leaving open the corresponding
lower bounds. \cite{carpentier2011upper} proposed a different algorithm and analyzed additional settings for
$\mathcal H$, including sub-Gaussian distributions with a known sub-Gaussian parameter bound and Gaussian
distributions; the lower-bound question remained unresolved.

\paragraph{}More recently, \cite{aznag2023an} obtained the near-optimal worst-case upper bound
$\tilde O\!\left(\sqrt{d/T}\right)$ for this objective and provided the first lower-bound results in
the Gaussian case. In particular, they showed that a suitably normalized isotropic Gaussian instance attains
regret on the order of $\tilde \Omega\!\left(\sqrt{d/T}\right)$, matching the worst-case upper bound up to
logarithmic factors. To the best of our knowledge, these are \emph{the only existing lower bounds for the problem we are studying}. Our main theorem provides a local minimax lower bound for any hypothesis class $\mathcal H$ with finite variance,
expressed in terms of a one-arm information curvature $\mathrm{VLC}_\rho(\cdot\mid\mathcal H)$ and the variance's heteroscedasticity. This yields a unified lower-bound framework for this problem.

\paragraph{}From a technical point of view, our work can be embedded within the information geometry literature \citep{nielsen2020elementary}. Recently, information geometry \citep{atz2021geometric,boguna2021network,isert2023structure,10.1145/3708498} has provided promising tools for tackling non-linearity aspects in learning problems. Our core idea, \textit{instance generation}, essentially assumes that $\mathcal{H}$ can be embedded in a statistical manifold. This establishes a natural correspondence between the set of instances $\mathcal{H}$, and a (locally) Euclidean space where calculations are more interpretable. In turn, this reduces the problem of finding a lower bound to optimizing over a (locally) Euclidean space.

\subsection{Paper Structure and Notation}
\paragraph{Paper structure.}
Section~\ref{sec:vlc} defines the Variance Local Curvature (VLC) and states our main local minimax lower bound (Theorem \ref{thm:main}),
together with its various implications. Section~\ref{sec:smooth} studies the particular case when the hypothesis class $\mathcal{H}$ represents smooth classes.
Section~\ref{sec:proof} provides a proof outline, highlighting the two core ideas (loss-induced geometry and
representation-based instance generation). Proofs are deferred to the appendix.

\paragraph{Notation.}
For an integer $n\ge 1$, we write $[n]:=\{1,\dots,n\}$. Vectors are denoted in bold, and $\one=(1,\dots,1)\in\R^d$
is the all-ones vector. For $\vct x,\vct y\in\R^d$, $\langle \vct x,\vct y\rangle$ denotes the Euclidean inner
product and $\|\vct x\|_p$ the $\ell_p$ norm. We write $\vct x^{\odot 2}$ for the element-wise square of $\vct x$.
For a vector $\vct u\in\R^d$, we write $\vct u^\perp:=\{\vct x\in\R^d:\langle \vct x,\vct u\rangle=0\}$. For a finite
set $S$, $|S|$ denotes its cardinality. We use $\Rpos:=(0, +\infty)$ and $\Rplus:=[0,\infty)$. For a random variable $X$, we write $\E[X]$ for its expectation and
$\Var(X)$ for its variance. For two distributions $P,Q$, $\KL(P\|Q)$ denotes the Kullback--Leibler divergence. When we write
$\tilde O(\cdot)$ (instead of $O(\cdot)$), polylogarithmic factors in $T$ are suppressed.
\section{Variance Local Curvature and Main Theorem}\label{sec:vlc}

\begin{definition}[Variance Local Curvature]\label{def:VLC-onearm}
Fix $(\sigma,\rho)\in\Rpos\times(0,1)$ and define
\[
B_\rho(\sigma)
\;:=\;
\Bigl\{v\in\Rpos:\ \Bigl|\frac{v^2-\sigma^2}{\sigma^2}\Bigr|\le \rho\Bigr\}.
\]
Let $\mathfrak{S}_\rho(\sigma)$ denote the set of measurable selectors
\[
\mathfrak{S}_\rho(\sigma)
\;:=\;
\Bigl\{D:B_\rho(\sigma)\to\mathcal H\ \text{measurable}:\ \Var(D(v))=v^2\ \ \forall v\in B_\rho(\sigma)\Bigr\}.
\]

Equivalently, $B_\rho(\sigma)=\bigl[\sigma\sqrt{1-\rho},\,\sigma\sqrt{1+\rho}\,\bigr]\cap\Rpos$ is an interval of
admissible standard deviations around $\sigma$.

The $\rho$-Variance Local Curvature is
\begin{equation}\label{eq:def-VLC}
\mathrm{VLC}_{\rho}(\sigma\mid\mathcal{H})
\;:=\;
\inf_{D\in \mathfrak{S}_\rho(\sigma)}\;
\sup_{v_a,v_b\in B_\rho(\sigma)}
\frac{\KL(D(v_a)\,\|\,D(v_b))}
{\Bigl(\frac{v_a^2-v_b^2}{\sigma^2}\Bigr)^2}
\;\in\Rplus \cup\{+\infty\},
\end{equation}
with the convention $\mathrm{VLC}_{\rho}(\sigma\mid\mathcal{H})=+\infty$ if $\mathfrak{S}_\rho(\sigma)=\emptyset$. Moreover, the Variance Local Curvature is defined as
\begin{equation*}
    \mathrm{VLC}(\sigma\mid\mathcal{H}) := \mathrm{VLC}_{0^+}(\sigma\mid\mathcal{H}) = \limsup_{\rho \to 0^+}\mathrm{VLC}_{\rho}(\sigma\mid\mathcal{H}).
\end{equation*}
\end{definition}

Although Definition~\ref{def:VLC-onearm} optimizes over all measurable selectors, the resulting curvature is usually
easy to evaluate, as it admits simple closed forms in many standard settings (see Table~\ref{tab:vlc-fisher}).

\begin{remark}[Examples]\label{rem:vlc-examples}
To build intuition for $\mathrm{VLC}$, we look at the following two extremes.
\emph{(i) A rich class with vanishing curvature.} Let
$\mathcal H=\{(1-\varepsilon)\delta_0+\varepsilon\delta_C:\varepsilon\in(0,1)\}$ with $C\gg 0$. A distribution here
has variance $\varepsilon(1-\varepsilon)C^2$, so the variance can be moved by an arbitrarily small change in
$\varepsilon$ once $C$ is large, while $\KL$ between two such distributions depends only on $(\varepsilon,\varepsilon')$
and stays bounded. Variance is thus ``cheap'' to change in $\KL$, and $\mathrm{VLC}(\sigma\mid\mathcal H)\to 0$ as
$C\to\infty$.
\emph{(ii) A rigid class with curvature bounded below.} Let $\mathcal H=\{\mathrm{Normal}(0,\sigma^2):\sigma>0\}$ be
the centered Gaussian family. By Lemma~\ref{lem:VLC-Fisher} (see the Normal row of Table~\ref{tab:vlc-fisher}),
$\mathrm{VLC}(\sigma\mid\mathcal H)=\tfrac14$ for every $\sigma>0$.
\end{remark}

\subsection{Main Result}\label{sec:main-result}

Our main theorem lower bounds the local minimax value $\mathcal V_{\rho,\tau}(\vct{\sigma})$ in terms of the
one-arm complexities $\mathrm{VLC}_{\rho}(\sigma_k\mid\mathcal H)$ and the heteroscedasticity of the standard deviations $\vct \sigma$. We introduce the heteroscedasticity index, which describes how uneven the randomness is between arms.
\begin{equation*}
    \hts_k(\vct \sigma) := \frac{\sigma_k^2}{\|\vct \sigma\|_2^2}\left(1-\frac{\sigma_k^4}{\|\vct\sigma^2\|_2^2}\right).
\end{equation*}
We are now ready to state our main theorem.

\begin{theorem}[Local minimax lower bound]\label{thm:main}
There exist universal constants $C>0$ such that for any $\vct{\sigma}\in\R_+^d$, there exists a minimal resolution $\rho^* \propto \sqrt{\frac{d}{T}}$ such that for any $\tau>0$, and any $\rho \;\ge\; \rho^\star$, the local minimax value satisfies
\[
\mathcal V_{\rho,\tau}(\vct{\sigma})
\;\ge\;
C\Bigl(1 - \epsilon(\rho, \tau)\Bigr)\sqrt{\frac{\|\vct \sigma\|_0}{T}}\,
\sqrt{\sum_{k=1}^d
\frac{\hts_k(\vct \sigma)}
{\mathrm{VLC}_\rho(\sigma_k\mid\mathcal H)}},
\]
where $\epsilon(\rho, \tau)$ vanishes when $(\rho, \tau) \to (0, 0)$.
\end{theorem}

Our lower bound decomposes the difficulty of mean estimation into three orthogonal ingredients: a \emph{budget}
term, a \emph{heteroscedasticity} term, and a \emph{structural} term.
\[
\underbrace{\bigl(1-\epsilon(\rho, \tau)\bigr)}_{\substack{\text{Error term}}}\;\cdot\;\underbrace{\sqrt{\frac{\|\vct \sigma\|_0}{T}}}_{\substack{\text{Budget}}}\;
\sqrt{\sum_{k=1}^d
\underbrace{\hts_k(\vct \sigma)}_{\substack{\text{Heteroscedasticity}}}\;
\underbrace{\frac{1}{\mathrm{VLC}_\rho(\sigma_k\mid\mathcal H)}}_{\substack{\text{Model curvature}}}}
.
\]
The budget contribution is standard and almost ubiquitously present in learning theory. The term $\|\vct \sigma\|_0$ represents the \emph{effective dimentionality} of the problem. Moreover, the multiplicative correction $1-\epsilon(\rho, \tau)$ captures two sources of approximation: $\rho$ controls
the locality scale (finer perturbations as $\rho\downarrow 0$), and $\tau$ restricts attention to relevant
policies. In the vanishing regime $\rho\downarrow 0$ and $\tau\downarrow 0$, this correction tends to $1$.
We discuss these regime conditions and their interpretation in Section \ref{sec:discussions}.

The heteroscedasticity weights $\hts_k(\vct \sigma)$ describe how unevenly the randomness is spread across arms.
They are scale invariant: multiplying $\vct{\sigma}$ by a constant does not change these weights,
but redistributing mass across coordinates does. When the variances are equal (more homogeneous 
uncertainty), the learner must hedge broadly, since many arms can potentially dominate the max-risk
objective. This is analogous to worst-case rates in learning where difficulty scales with the number of
coordinates that can be simultaneously active: a diffuse signal forces exploration across many directions.
By contrast, when uncertainty is concentrated on a few arms, the problem becomes effectively lower-dimensional:
only a small subset can plausibly determine the maximum, and a policy can focus its budget there. In this sense, concentrated uncertainty is easier because it reduces the number of arms that must be balanced at comparable precision to control the max-risk objective; equivalently, $\|\vct \sigma\|_0 \cdot \sum_k \hts_k(\vct \sigma)$ acts as an effective number of relevant arms.

The structural term $\mathrm{VLC}_\rho(\sigma_k\mid\mathcal H)$ measures how much \emph{information} is created by
a local change in variance within the model class. It is a curvature: larger $\mathrm{VLC}_\rho$ means that
nearby variance levels are strongly separated in KL, so variance is easier to learn and allocation mistakes
become less persistent; smaller $\mathrm{VLC}_\rho$ means the class is richer in ways that allow variance to
change while producing little KL signal, making it intrinsically harder to discriminate nearby variance
profiles. This mirrors familiar phenomena in statistical learning and nonparametrics: in rigid, well-specified
models, local parameters are easily identifiable, while in richer classes the
same functional can have weaker effective curvature and require substantially more data. Here this manifests
directly in the factor $1/\sqrt{\mathrm{VLC}_\rho}$: weaker curvature enlarges the ``local indistinguishability''
region, forcing larger regret even when the budget and heteroscedasticity are held fixed. A more detailed discussion on the $\mathrm{VLC}$ for smooth classes $\mathcal{H}$ is deferred to Section~\ref{sec:smooth}.

The bound in Theorem \ref{thm:main} separates the effect of the class $\mathcal{H}$ from the effect of the instance $\vct D$.
The class $\mathcal H$ enters through the operator $\mathrm{VLC}(\,\cdot\mid\mathcal H)$, which assigns to each
variance level $\sigma$ a local information curvature. The instance $\vct D$ influences the bound in two local
ways: first through the \emph{heteroscedasticity profile} $\hts_k(\vct\sigma(\vct D))$, which quantifies how much
arm $k$ contributes to the \emph{directional} uncertainty relevant for the allocation decision, and second through
the evaluation of the curvature operator at the realized variance levels, $\mathrm{VLC}(\sigma_k(\vct D)\mid\mathcal H)$.
Concretely, the contribution of arm $k$ takes the form
\[
\frac{\sqrt{\hts_k(\vct\sigma(\vct D))}}{\sqrt{\mathrm{VLC}(\sigma_k(\vct D)\mid\mathcal H)}},
\]
so $\mathcal H$ determines the curvature profile $\sigma\mapsto \mathrm{VLC}(\sigma\mid\mathcal H)$, while $\vct D$ determines
where this profile is evaluated and how the heteroscedasticity (and thus the effective difficulty) is distributed across arms.

\subsection{Discussions}\label{sec:discussions}
\paragraph{Worst-case bounds and interpolation over subclasses.}
For $\rho\in(0,1)$, call a distribution $D\in\mathcal H$ \emph{$\rho$-interior} if its
standard deviation $\sigma(D)=\sqrt{\Var(D)}$ admits a nonempty local selector set, i.e.,
\[
D\in\mathcal H^\circ_\rho
\qquad\Longleftrightarrow\qquad
\mathfrak S_\rho\big(\sigma(D)\big)\neq\emptyset.
\]
Given any subclass $\mathcal G\subseteq \mathcal H^\circ_\rho$, define its worst-case $\rho-$VLC by
\[
\mathrm{VLC}^\star_\rho(\mathcal G\mid\mathcal H)
\;:=\;
\inf_{D\in\mathcal G}\ \mathrm{VLC}_\rho\big(\sigma(D)\mid\mathcal H\big)
\ \in\ \Rplus\cup\{+\infty\}.
\]
Then, for each $(D, \ldots, D) \in \mathcal{G}^d$, we have
\begin{equation*}
    \sum_{k=1}^d 
\frac{\hts_k(\sigma \cdot \vct 1)}{\sqrt{\mathrm{VLC}_{\rho}(\sigma(D)\mid\mathcal H)}} = \frac{\|\hts(\vct 1)\|_1}{\sqrt{\mathrm{VLC}_{\rho}(\sigma(D)\mid\mathcal H)}} = \frac{1 - \frac{1}{d}}{\sqrt{\mathrm{VLC}_{\rho}(\sigma(D)\mid\mathcal H)}}.
\end{equation*}

Applying Theorem~\ref{thm:main} and the equality above yields the worst case lower bound over the $\rho-$interior subclass $\mathcal{G}^d$
\[
\sup_{\vct D \in \mathcal G^d}\mathcal V_{\rho,\tau}\big(\vct \sigma (\vct D)\big)
\;\geq\;
\Omega\!\left(
(1 - \epsilon(\rho, \tau))\left(1 - \frac{1}{d}\right)\sqrt{\frac{d}{\mathrm{VLC}^\star_\rho(\mathcal G\mid\mathcal H) T}}
\right).
\]
In particular, the bound in Theorem \ref{thm:main} is modular in the instance class, as any restriction to a subclass
$\mathcal G\subseteq \mathcal H^\circ_\rho$ simply replaces the global curvature constant by its worst-case
value over $\mathcal G$, yielding a graded family of guarantees. Finally, taking $\mathcal G=\mathcal H^\circ_\rho$ yields a worst-case lower bound over all
$\rho$-interior instances, i.e., 
\[
\sup_{\vct D \in (\mathcal H^\circ_\rho)^d}\mathcal V_{\rho,\tau}\big(\vct \sigma (\vct D)\big)
\;\geq\;
\Omega\!\left(
(1 - \epsilon(\rho, \tau))\left(1 - \frac{1}{d}\right)\sqrt{\frac{d}{\mathrm{VLC}^\star_\rho( \mathcal H^\circ_\rho \mid\mathcal H) T}}
\right).
\]

\paragraph{Resolution limit and vanishing-regime.}
When $\rho\ll \sqrt{d/T}$, the perturbations are below the statistical resolution of a horizon-$T$ experiment
with $d$ arms, and the information-theoretic reduction used in the proof of Theorem~\ref{thm:main} becomes
insensitive to such fine differences. In this sense,
\[
\rho^\star=\Theta\left(\sqrt{\tfrac{d}{T}}\right)
\]
marks the natural resolution limit of our proof technique. More broadly, while the theorem statement is valid for any tolerance $\tau$ and any locality radius $\rho$ (above
feasibility), it becomes non-vacuous only in regimes where both $\rho\to 0$ and $\tau\to 0$. Sending
$\rho\to 0$ corresponds to sending the granularity of the local discrimination problem to zero: we seek a bound
that remains nontrivial even under arbitrarily fine perturbations. 

The tolerance $\tau\to 0$ may appear unusual,
but it is conceptually natural: it restricts attention to learnable regimes in which there exists at least one
policy whose worst-case normalized regret over the local class vanishes. If $\tau$ is not small, the local minimax
formulation allows policies that perform poorly even on the local class, and the value $\mathcal V_{\rho,\tau}$
becomes less connected to the asymptotic learning question the lower bound is meant to capture. 

If $D\in\mathcal H^\circ_\rho$, then $D\in\mathcal H^\circ_{\rho'}$
for every $\rho'\in(0,\rho)$, since any selector on $B_\rho(\sigma(D))$ restricts to a selector on the smaller
neighborhood $B_{\rho'}(\sigma(D))$. Motivated by this ``nested balls'' property, define the interior of
$\mathcal H$ by
\[
\mathcal H^\circ
\;:=\;
\bigcap_{\rho\in(0,1)}\mathcal H^\circ_\rho.
\]
In the regime $\sqrt{\frac{d}{T}}\downarrow 0$,  $\tau\downarrow 0$, Theorem~\ref{thm:main} yields a non-vacuous asymptotic lower bound for every
interior instance. Formally, for any $\vct D_n\in(\mathcal H^\circ)^{d_n}$ and any sequence of $(\rho_n,\tau_n,T_n,d_n)$ satisfying
$\rho_n\ge\rho_n^\star$ and $\tau_n+\rho_n\to 0$, we have
\[
\limsup_n
\frac{\sqrt{T_n/d_n}\ \mathcal V_{\rho_n,\tau_n}\!\big(\vct\sigma(\vct D_n)\big)}
{\sqrt{\sum_{k=1}^{d_n} 
\frac{\hts_k(\vct \sigma(\vct D_n))}{\mathrm{VLC}_{\rho_n}(\sigma_k(\vct D_n)\mid\mathcal H)}}}
\;\ge\; C > 0,
\]
where $C$ is the universal constant from Theorem \ref{thm:main}.
This is the sense in which the theorem provides a ``true'' local minimax lower bound.

We emphasize that the locality radius and tolerance vanish at a rate tied to the horizon: the bound is informative
along \emph{schedules} with $\rho_T\gtrsim\sqrt{d/T}$ and $\tau_T\to 0$, rather than at a single fixed
$(\rho,\tau,T)$. Corollary~\ref{cor:fixed-instance} below specializes this to one fixed instance as $T\to\infty$. We first record the simplest instantiation, for a fixed instance and a fixed dimension as $T\to\infty$, before the
general worst-case version.
\begin{corollary}[Fixed instance, $T\to\infty$]\label{cor:fixed-instance}
Fix the dimension $d$ and an interior instance $\vct D\in(\mathcal H^\circ)^d$. For any schedule $(\rho_T,\tau_T)$
with $\rho_T\ge\rho_T^\star=\Theta(\sqrt{d/T})$ and $\tau_T\to 0$,
\[
\liminf_{T\to\infty}\
\frac{\sqrt{T/d}\;\mathcal V_{\rho_T,\tau_T}\big(\vct\sigma(\vct D)\big)}
{\sqrt{\sum_{k=1}^{d}\hts_k(\vct\sigma(\vct D))\big/\mathrm{VLC}(\sigma_k(\vct D)\mid\mathcal H)}}
\;\ge\;C\;>\;0,
\]
where $C$ is the universal constant of Theorem~\ref{thm:main}. Equivalently, for a fixed instance the local minimax
value obeys $\mathcal V_{\rho_T,\tau_T}(\vct\sigma(\vct D))=\Omega\!\big(\sqrt{d/T}\big)$, with the implied constant
set by the heteroscedasticity and curvature of $\vct D$.
\end{corollary}
\begin{corollary}[Informal]\label{cor:informal-worstcase}
Assume the problem is learnable at rate $\bar v \to 0$ in an asymptotic regime with $d/T\to 0$. Then
Theorem~\ref{thm:main} implies the worst-case interior lower bound
\[
\sup_{\vct D\in(\mathcal H^\circ)^d}\ \mathcal V_{\Theta\left(\sqrt{\frac{d}{T}}\right),\bar v}\!\big(\vct\sigma(\vct D)\big)
\;\ge\;
\Omega\!\left(\sqrt{\frac{d}{T}}\ \frac{1}{\sqrt{\mathrm{VLC}(\mathcal H^\circ\mid\mathcal H)}}\right).
\]

\end{corollary}
\section{VLC for Smooth Classes}\label{sec:smooth}
This section characterizes the VLC in a smooth regime where the KL divergence
admits a quadratic expansion \emph{along variance perturbations}. Concretely, we show that the VLC is a reparametrization of a Fisher information type quantity.
\subsection{VLC and Fisher Information}\label{sec:fisher}
We show that $\mathrm{VLC}$ can be written as a re-parametrization of the Fisher information. Our Fisher quantity is defined as the \emph{best} (i.e., most
favorable for minimizing KL at fixed variance level) local curvature over all realizations in $\mathcal H$ with
the prescribed variance.

\begin{definition}[Variance-Fisher information]\label{def:I-variance}
For a variance level $\theta \in \Rplus$, define the \emph{variance-Fisher information} at $\theta$ by
\[
I_{\text{Var}}(\theta)
\;:=\;
2 \liminf_{\Delta\to 0}\;
\inf_{\substack{P\in\mathcal{H}:\\ \Var(P)=\theta}}
\inf_{\substack{Q\in\mathcal{H}:\\ \Var(Q)=\theta+\Delta}}
\frac{\KL(Q\,\|\,P)}{\Delta^{2}}
\;\in\;[0,+\infty].
\]
\end{definition}
Equivalently, $I_{\text{Var}}(\theta)$ is the smallest quadratic KL curvature (in the variance coordinate) that
can occur among \emph{all} distributions in $\mathcal H$ with variance $\theta$, allowing any nuisance features
of the distribution (e.g., mean or higher moments) to vary as $\theta$ is perturbed.
\begin{assumption}[Local quadratic KL expansion in the variance coordinate]\label{as:local-KL}
Fix an instance $D \in \mathcal{H}$ and its variance level $\theta = \text{Var}(D)$ with $0<I_{\text{Var}}(\theta)<\infty$. Assume that for every $\varepsilon>0$
there exists $\delta_\varepsilon>0$ such that for all $|\Delta|\le \delta_\varepsilon$,
\begin{equation}\label{eq:local-KL-bounds}
\Bigl(\frac{I_{\text{Var}}(\theta)}{2}-\varepsilon\Bigr)\Delta^2
\;\le\;
\inf_{\substack{P\in\mathcal{H}:\\ \Var(P)=\theta}}
\inf_{\substack{Q\in\mathcal{H}:\\ \Var(Q)=\theta+\Delta}}
\KL(Q\,\|\,P)
\;\le\;
\Bigl(\frac{I_{\text{Var}}(\theta)}{2}+\varepsilon\Bigr)\Delta^2.
\end{equation}
Moreover, assume the upper bound is attained (up to $o(\Delta^2)$) by a $C^2$ curve
$\{P_\vartheta:\vartheta\in(\theta-\eta,\theta+\eta)\}\subset\mathcal{H}$ with $\Var(P_\vartheta)=\vartheta$.
\end{assumption}

Assumption~\ref{as:local-KL} is a standard smoothness condition: locally, the KL divergence between two nearby
variance levels is quadratic in the variance increment, with curvature $I_{\text{Var}}(\theta)/2$ after optimizing
over nuisance directions within $\mathcal H$. Lemma \ref{lem:VLC-Fisher} below shows that $\mathrm{VLC}_\rho$ is the corresponding
scale-free curvature that appears in our minimax regret bounds. Its proof is deferred to Appendix \ref{sec:proof-aux}.

\begin{lemma}[VLC and variance-Fisher information]\label{lem:VLC-Fisher}
Under Assumption~\ref{as:local-KL}, there exists $\rho_0\in(0,1)$ such that for all $\rho\in(0,\rho_0)$,
\[
\mathrm{VLC}_{\rho}(\sigma\mid\mathcal{H})
\;=\;
\frac{\sigma^{4}}{2}\, I_{\text{Var}}(\sigma^2)\,\bigl(1+o_\rho(1)\bigr),
\]
where $o_\rho(1)\to 0$ as $\rho\downarrow 0$. In particular, $\mathrm{VLC}(\sigma\mid\mathcal{H})=\frac{\sigma^{4}}{2}\, I_{\text{Var}}(\sigma^2)$.
\end{lemma}

In summary, in the smooth regime the difficulty of learning local variance profiles is controlled by the scalar
curvature $I_{\text{Var}}(\sigma^2)$, and VLC provides a scale-free normalization that is directly comparable across
models. In particular, whenever $I_{\text{Var}}(\sigma^2)$ is available in closed form, one immediately obtains an explicit formula
for $\mathrm{VLC}(\sigma\mid\mathcal H)$. Table~\ref{tab:vlc-fisher} records several standard examples.
\begin{table}[t]
\centering
\renewcommand{\arraystretch}{1.25} % more vertical space between rows
\begin{tabular}{@{}llll@{}}
\toprule
Class of distributions $\mathcal H$ &
Variance $\sigma^2$ &
$I_{\text{Var}}(\sigma^2)$ &
$\mathrm{VLC}(\sigma\mid\mathcal H)$ \\
\midrule
$\{\mathrm{Normal}(\mu,\theta):\mu\in\R,\ \theta>0\}$ &
$\theta$ &
$\dfrac{1}{2\sigma^{4}}$ &
$\dfrac14$ \\
\addlinespace[0.4em]
$\{\mathrm{Laplace}(\mu,b):\mu\in\R,\ b>0\}$ &
$2b^2$ &
$\dfrac{1}{4\sigma^{4}}$ &
$\dfrac18$ \\
\addlinespace[0.4em]
$\{\mathrm{Exponential}(\lambda):\lambda>0\}$ &
$\dfrac{1}{\lambda^2}$ &
$\dfrac{1}{4\sigma^{4}}$ &
$\dfrac18$ \\
\addlinespace[0.4em]
$\{\mathrm{Gamma}(\alpha,\beta):\beta>0\}$ (fixed $\alpha>0$) &
$\dfrac{\alpha}{\beta^2}$ &
$\dfrac{\alpha}{4\sigma^{4}}$ &
$\dfrac{\alpha}{8}$ \\
\addlinespace[0.4em]
$\{\mathrm{Bernoulli}(p):p\in(0,1)\}$ &
$p(1-p)$ &
$\dfrac{1}{\sigma^{2}(1-4\sigma^{2})}$ &
$\dfrac{\sigma^{2}}{2(1-4\sigma^{2})}$ \\
\bottomrule
\end{tabular}\caption{Examples of variance-Fisher information (minimized over $\mathcal H$ at a fixed variance level $\sigma^2$) and the corresponding VLC via Lemma~\ref{lem:VLC-Fisher}.}
\label{tab:vlc-fisher}
\end{table}
\subsection{Distance to Optimality Under The Smoothness Assumption}

Under Assumption~\ref{as:local-KL}, we benchmark Theorem~\ref{thm:main} against the strongest available general upper bound, namely Theorem~4.1 of~\cite{aznag2025activelearningframeworkmultigroup}. This comparison shows that our guarantee is near-optimal (up to an $O(\sqrt{\log T})$ factor) in a broad class of structured regimes, and we exhibit the discrepancies in the general case. Moreover, a side-by-side reading of the proof of Theorem~4.1 and our argument highlights the main sources of slack in the current upper-bound technique, and suggests concrete, high-level directions for sharpening it. 

% (e.g., exploiting problem geometry beyond coordinatewise confidence widths).

Their bound relies on a general UCB framework, which we instantiate by constructing an explicit UCB for the variance and then invoke their regret guarantee directly.

\paragraph{A computable UCB for the variance.}
Assume each arm belongs to a smooth one-parameter subfamily $\{P_\vartheta:\vartheta\in\Theta\}\subset\mathcal H$.
Write $\Var(P_\vartheta)$ for the variance under parameter $\vartheta$. For an arm with $n$ samples, let
$\widehat\vartheta_n$ be an estimator of $\vartheta$ (e.g., the MLE) and set $\beta_t:=\log(T)$. Define the
(data-driven) upper confidence bound on the variance by likelihood-ratio inversion:
\begin{equation}\label{eq:ucb-var}
(\sigma_n^{+})^{2}
\;:=\;
\sup\Bigl\{\Var(P_{\vartheta'}) : \vartheta'\in\Theta,\ n\,\KL\!\bigl(P_{\widehat\vartheta_n}\,\|\,P_{\vartheta'}\bigr)\le \beta_t\Bigr\}.
\end{equation}

\paragraph{Allocation rule (Variance-UCB).}
Initialize by pulling each arm once. For $t>d$, let $n_{k,t}$ be the number of pulls of arm $k$ up to time $t$,
and let $\sigma_{k,t}^{+}$ be the bound \eqref{eq:ucb-var} computed from the samples of arm $k$.
At time $t+1$, pull an arm maximizing
\[
k_{t+1}\in\arg\max_{k\in[d]}\ \frac{(\sigma_{k,t}^{+})^2}{n_{k,t}}.
\]

\paragraph{Upper bound under smoothness.}
Under Assumption~\ref{as:local-KL}, the confidence radius induced by
\eqref{eq:ucb-var} has a Fisher scaling, hence a $\mathrm{VLC}$ scaling. Consequently, applying Theorem 4.1 in 
\cite{aznag2025activelearningframeworkmultigroup} implies the following upper bound
\begin{equation*}
\Reg(\pi,\vct{D})
=
O\!\left(\frac{\sqrt{\log T}}{\sqrt{T}}\,
\sum_{k=1}^d
\frac{\sigma_k(\vct{D})}{\|\vct{\sigma}(\vct{D})\|_2}\cdot
\frac{1}{\sqrt{\mathrm{VLC}(\sigma_k(\vct{D})\mid\mathcal H)}}\right).
\end{equation*}

\paragraph{Discrepancies between the upper and lower bounds.}

To make the gap between the best known upper bound and our lower bound transparent, we rewrite
the leading term in Theorem~4.1 of~\cite{aznag2025activelearningframeworkmultigroup} in our notation.
Applying Cauchy--Schwarz to the right-hand side and introducing the normalized variance weights
\[
\hts_k^{+}(\vct\sigma)
\;:=\;
\frac{\sigma_k^2}{\|\vct\sigma\|_2^2},
\]
we obtain
\begin{equation}\label{eq:upper-smooth}
    \Reg(\pi,\vct D)
    \;=\;
    O\!\left(
    \sqrt{\frac{\log T\;\|\vct\sigma\|_0}{T}}\;
    \sqrt{\sum_{k=1}^d
    \frac{\hts_k^{+}(\vct\sigma (\vct D))}{\mathrm{VLC}(\sigma_k(\vct D)\mid\mathcal H)}}
    \right).
\end{equation}

Up to polylogarithmic factors, the only mismatch between~\eqref{eq:upper-smooth} and the lower bound in
Theorem~\ref{thm:main} is the appearance of $\hts^{+}$ in place of $\hts$.
By construction, $\hts_k(\vct\sigma)\le \hts_k^{+}(\vct\sigma)$ for every coordinate $k$, so the bound
\eqref{eq:upper-smooth} may be loose precisely in regimes where this inequality is far from tight. The key distinction is that $\hts^+_k(\vct\sigma):=\sigma_k^2/\|\vct\sigma\|_2^2$ is \emph{only} the oracle Neyman weight: it measures the share of total variance carried by arm $k$, but it does not encode how this share interacts with the \emph{rest} of the instance. In contrast, our heteroscedasticity term
\[
\hts_k(\vct\sigma)
\;=\;
\hts^+_k(\vct\sigma)\Bigl(1-\frac{\sigma_k^4}{\|\vct\sigma^2\|_2^2}\Bigr)
\]
penalizes directions in which the variance profile is nearly one-dimensional: it is large only when arm $k$ carries non-negligible variance \emph{and} there remains substantial ``variance-squared mass'' outside coordinate $k$. Consequently, when the variance is reasonably spread across the active arms (no single $\sigma_k^4$ dominates $\|\vct\sigma^2\|_2^2$), the correction factor is $\Theta(1)$ and $\hts$ and $\hts^+$ are comparable (indeed they coincide up to constants in the homogeneous-on-support regime). By contrast, in spiky regimes where many arms have vanishing variance and one (or a few) arms dominate $\|\vct\sigma^2\|_2$, the factor $1-\sigma_k^4/\|\vct\sigma^2\|_2^2$ collapses for the dominant coordinates, driving $\hts$ (and thus the intrinsic difficulty) to $0$ in the extreme case of a single nonzero variance, whereas $\hts^+$ remains order-one. This is exactly the phenomenon $\hts$ captures that $\hts^+$ misses: the \emph{directional} uncertainty relevant for identifying the optimal allocation vanishes as the instance approaches the boundary of the simplex, even though the largest variance weight itself does not. This behavior is also algorithmically natural: if all but one arm are (nearly) deterministic, then after a handful of samples one can certify their variance is (close to) zero and essentially stop allocating budget to them. The remaining uncertainty is concentrated on the single noisy arm, so an optimal procedure should quickly focus its pulls there, yielding vanishing ``allocation identification'' regret in that boundary regime.

\paragraph{Guidelines for sharpening future upper bounds.} This comparison suggests concrete guidelines for sharpening future upper-bound analyses. First, rather than controlling each $\sigma_k$ independently and then union-bounding over $k$, one should analyze the \emph{normalized} variance profile (equivalently, the Neyman weight vector) and its fluctuations in directions tangent to the simplex. Moreover, one should allow confidence budgets and error metrics that are \emph{instance-weighted} (e.g., weighted by $\lambda^\star$) rather than uniform across arms, so that arms with vanishing variance can be eliminated quickly. Currently, such refinements are not captured by the current upper bound strategies.
\section{Proof Outline}\label{sec:proof}

\paragraph{Overview.}Lower bounds are obtained by randomizing over instances that are close enough to generate nearly indistinguishable data, yet far enough that their optimal allocations differ substantially, forcing any policy to make costly mistakes. While we follow this template here as well, we introduce two key ingredients that will be crucial in overcoming two separate obstacles.

\textbf{(i) Identifying the geometry:}
the first obstacle is to identify the right geometry on the set of decisions: we need a notion of distance that is compatible with the regret. In the classic multi-armed bandit
setting with additive regret, one typically has
\[
\Reg_T \;=\; \sum_{k=1}^d \E[n_k]\cdot \Delta_k = \E\left[\langle \vct n, \vct \Delta \rangle\right],
\]
where $\vct \Delta$ is the suboptimality gap vector. In particular, regret is linear in the sampling counts
$\vct n$, which makes separation essentially linear. In other words, the right geometry that measures between two different decisions is the euclidean geometry. In our setting, the normalized excess risk takes the max form
\[
\frac{R(\vct n;\vct\sigma)}{R^\star(\vct\sigma)}-1
\;=\;
\frac{\max_{k\in[d]}\ \sigma_k^2/n_k}{\|\vct\sigma\|_2^2/T}-1,
\]
which is non-additive and non-linear in $\vct n$, and therefore the right geometry is not euclidean. We show that the right geometry for our problem is the $\ell_1$ geometry: separation between decisions is measured (up to a universal constant) by $\|\cdot\|_1$.

\textbf{(ii) Identifying the adversarial instance:} the second obstacle is to construct the adversarial family in a way that is both local and analyzable. In an additive objective one can often propose a simple perturbation family and check that it is hard. Here this is much less clear: the hardness depends on anisotropy coming from the instance and from the class $\mathcal H$ (through $\mathrm{VLC}_\rho$), so there is no single ``obvious'' hard direction. Our approach is therefore not to guess a hard family, but to parameterize \emph{all} local hard families we care about through a representation map. Concretely, we encode nearby decisions by a hypercube code $x\in\mathcal X$ and a linear map $A$, and view the adversary as choosing the representation $x\mapsto Ax$. This renders the lower bound calculable: the same representation controls both distinguishability (via average KL) and decision separation (via the $\ell_1$ geometry), and the problem reduces to a tractable matrix optimization. This viewpoint is what makes the construction interpretable and what allows us to use sharp first-order calculations to handle anisotropy in closed form. In particular, it is this viewpoint that will allow us to capture the refined heteroscedasticity that we believe is relevant to the problem.

\textbf{From the two ideas to the bound.} Combining these ingredients gives the following pipeline, carried out
in full in Appendix~\ref{app:proof-overview}. (1) We fix the decision-relevant level set (the scale-free variance
profile). (2) Over this slice we select
maximally indistinguishable distributions, quantified by $\mathrm{VLC}_\rho(\cdot\mid\mathcal H)$. (3) We re-index
the resulting hard instances by nearby \emph{decisions}, turning the problem into a testing problem on the decision
set. (4) We embed a hypercube code $\mathcal X$ through a linear map $A$ and lower bound the regret by a
decision-mismatch (classification) cost measured in the $\ell_1$ geometry.

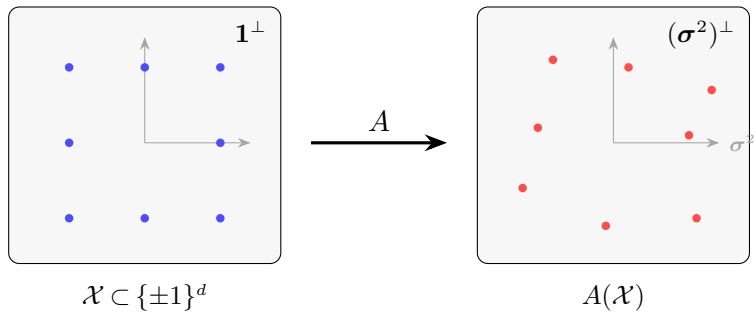
\begin{figure}[t]
\centering
\begin{tikzpicture}[>=Stealth]
  % --- left: hypercube code X in the hyperplane orthogonal to 1 ---
  \begin{scope}
    \draw[rounded corners, fill=gray!6] (-1.8,-1.6) rectangle (1.8,1.8);
    \node[anchor=north east] at (1.75,1.75) {\small $\one^{\perp}$};
    \draw[->,gray!70] (0,0) -- (1.4,0);
    \draw[->,gray!70] (0,0) -- (0,1.4);
    \foreach \p in {(-1,-1),(-1,1),(1,-1),(1,1),(-1,0),(1,0),(0,1),(0,-1)}
      {\fill[blue!70] \p circle (1.6pt);}
    \node at (0,-2.05) {\small $\mathcal X\subset\{\pm1\}^{d}$};
  \end{scope}
  % --- the linear map A ---
  \draw[->,very thick] (2.2,0) -- (4.0,0) node[midway,above] {$A$};
  % --- right: image A(X) in the hyperplane orthogonal to sigma^2 ---
  \begin{scope}[shift={(6.2,0)}]
    \draw[rounded corners, fill=gray!6] (-1.8,-1.6) rectangle (1.8,1.8);
    \node[anchor=north east] at (1.75,1.75) {\small $(\vct\sigma^{2})^{\perp}$};
    \draw[->,gray!70] (0,0) -- (1.4,0) node[right] {\scriptsize $\vct\sigma^{2}$};
    \draw[->,gray!70] (0,0) -- (0,1.4);
    \foreach \p in {(-1.2,-0.6),(-0.8,1.1),(1.1,-1.0),(1.3,0.7),(-1.0,0.2),(1.0,0.1),(0.2,1.0),(-0.1,-1.1)}
      {\fill[red!70] \p circle (1.6pt);}
    \node at (0,-2.05) {\small $A(\mathcal X)$};
  \end{scope}
\end{tikzpicture}
\caption{Illustration of the instance-generation process: a subset of the hypercube
$\mathcal X\subset\{\pm1\}^{d}$, lying in the hyperplane $\one^{\perp}$, is transformed through a linear map $A$
into a new set $A(\mathcal X)$ in the hyperplane $(\vct\sigma^{2})^{\perp}$. Each $A(x)$ is identified with a
generated instance $D(A,x)$.}
\label{fig:instance-generation}
\end{figure}

\section{Conclusions}\label{sec:closing}

Our results characterize the local minimax difficulty of adaptive allocation in the non-degenerate regime where
$\mathrm{VLC}_\rho(\sigma\mid\mathcal H)\in(0,\infty)$ at the relevant locality scales. Two natural extensions go
beyond this setting.

First, the curvature $\mathrm{VLC}_\rho$ is defined through a quadratic normalization in the variance coordinate. For extremely rich classes, KL separation under local variance perturbations can be smaller than quadratic, in
which case the quadratic curvature collapses and the canonical $\sqrt{T}$ scaling may no longer be the right resolution. Conversely, for classes that are too constrained, local variance neighborhoods may be infeasible or induce singular KL behavior. A refined local complexity notion that captures the correct power-law (or more general) modulus of KL with respect to variance perturbations would provide a principled way to treat both degeneracies within a single framework.

Second, the proof strategy we introduce (namely, an adapted geometry for decision separation, and a representation-based instance generation for constructing hard families) should extend to other non-additive objectives on allocations, where standard bandit lower-bound templates do not apply directly. Understanding the corresponding geometries and deriving sharp performance limits in these settings is, in its own right, an interesting direction for future work.

\ACKNOWLEDGMENT{A.A. and A.E. supported in part by NSF grant IIS-2147361. R.C. supported in part by NSF grants CNS-2138834 (CAREER) and IIS-2147361.}

\bibliographystyle{informs2014}
\bibliography{yourbibfile}
\newpage
\begin{APPENDICES}

\section{Proof Overview of Theorem \ref{thm:main}}\label{app:proof-overview}
In this section, we provide the road map to derive Theorem \ref{thm:main}, and defer the technical proofs for the intermediary lemmas to Appendix \ref{sec:proof-aux}.

The proof builds a lower bound by designing a collection of nearby instances that strategically confuse any learning policy. The idea is to make these instances close enough to be statistically indistinguishable but not so close that misidentifying them yields negligible regret. Achieving this balance requires controlling both their scale (how far they are) and their orientation (how they are arranged). We face mainly two obstacles:

\textbf{Obstacle (i): geometry.} The first challenge, that of scale, is addressed by introducing a problem-specific notion of distance that measures how costly it is to allocate resources incorrectly.

\textbf{Obstacle (ii): representation.} The second challenge, that of orientation, recognizes that even when instances are at the right distance, their relative positioning still matters: rotating or reshaping them can make the problem significantly harder. 

To systematically optimize these two aspects, we lift the search for hard instances into a structured space of configurations and reduce the lower bound construction to an explicit optimization tradeoff over this space. This yields a principled framework that identifies the most confusing yet consequential perturbations of the original problem, establishing the desired lower bound.

\subsection{Constructing the Lower Bound Framework}\label{app:framework}

To establish a lower bound on the local minimax value $\mathcal V_{\rho,\tau}(\vct{\sigma})$, our strategy is to
restrict the full, complex local class $\mathcal H_{\rho}(\vct{\sigma})$ to a smaller, more tractable, and maximally
challenging subset, that will be bijective to a ball inside the $d-$simplex $\Delta_d$. This construction proceeds in three main steps. First, we restrict attention to a
decision-relevant \emph{level set} to isolate the regret due to learning (rather than changes in the benchmark
risk) This is \textbf{Step 1}. Second, over this restricted set, we select distributions that are as statistically indistinguishable as
possible, a notion quantified by $\mathrm{VLC}_\rho(\cdot\mid\mathcal H)$. This is \textbf{Step 2}. Finally, we show that the resulting class
of hard instances can be re-parameterized by a simple geometric set of nearby decisions, which forms the basis of
the subsequent reduction to testing. This is \textbf{Step 3}.

\paragraph{Step 1: fixing the decision-relevant level set.}
The benchmark risk $R^\star(\vct{\sigma})=\|\vct{\sigma}\|_2^2/T$ depends only on the overall scale of $\vct{\sigma}$.
To separate regret due to learning from changes in the benchmark, we work on a fixed-scale slice by passing to the
scale-free variance profile
\[
\mathcal N(\vct{\sigma}) \;:=\; \frac{\vct{\sigma}^{\odot 2}}{\|\vct{\sigma}\|_2^2}\in\Delta_d.
\]
Fixing $\mathcal N(\vct{\sigma})$ fixes the decision-relevant profile while quotienting out global scale. For a radius
$\rho\in(0,1)$, we consider instances whose normalized profile lies in a $\rho$-neighborhood of $\mathcal N(\vct{\sigma})$:
\[
\tilde{\mathcal H}_{\rho}(\vct{\sigma})
\;:=\;
\Bigl\{\vct{D}\in\mathcal H^d:\ \bigl\|\mathcal N(\vct{\sigma}(\vct D))/\mathcal N(\vct{\sigma})-\one\bigr\|_2\le \rho, \quad R^\star(\vct{\sigma}(\vct{D})) = R^\star(\vct{\sigma}) \Bigr\}.
\]

\paragraph{Step 2: selecting maximally indistinguishable distributions (VLC selectors).}
To construct a hard family over $\tilde{\mathcal H}_{\rho}(\vct{\sigma})$, we must specify, for each arm $k$ and each nearby
standard deviation level $v_k$, a distribution in $\mathcal H$ with variance $v_k^2$. To make learning maximally
difficult, we choose these distributions to be as statistically similar as possible across nearby variances. This is
exactly what $\mathrm{VLC}_\rho(\sigma_k\mid\mathcal H)$ captures: it is the smallest achievable KL curvature under local
variance perturbations at scale $\rho$.

Concretely, for each arm $k$, let $\{D^{(i)}_{k}\}_{i\ge 1}$ be a sequence of measurable selectors on $B_\rho(\sigma_k)$ such that
\begin{equation}\label{eq:optimalrepresentation-app}
\sup_{v_a,v_b\in B_\rho(\sigma_k)}
\frac{\KL\!\bigl(D_{k}^{(i)}(v_a)\,\|\,D_{k}^{(i)}(v_b)\bigr)}
{\Bigl(\frac{v_a^2-v_b^2}{\sigma_k^2}\Bigr)^2}
\;\longrightarrow\;
\mathrm{VLC}_\rho(\sigma_k\mid\mathcal H),
\qquad i\to\infty.
\end{equation}
Given $\vct v\in\R_+^d$, define the induced $d$-arm instance
\[
D^{(i)}(\vct v)\;:=\;\bigl(D_{1}^{(i)}(v_1),\dots,D_{d}^{(i)}(v_d)\bigr)\in\mathcal H^d.
\]

\paragraph{Step 3: re-parameterizing hard instances by nearby decisions.}
Rather than working directly with variance vectors, we parameterize local alternatives by decisions in the simplex.
Let $\vct{\lambda}^\star:=\mathcal N(\vct{\sigma})$. Define the local set of decisions
\[
\Delta_{\rho}(\vct{\lambda}^\star)
\;:=\;
\Bigl\{\vct{\lambda}\in\Delta_d:\ \bigl\|\vct{\lambda}/\vct{\lambda}^\star-\one\bigr\|_2\le \rho\Bigr\}.
\]
For each $\vct{\lambda}\in\Delta_\rho(\vct{\lambda}^\star)$, we define a canonical representative on the same scale as
$\vct{\sigma}$ by
\[
\vct{\sigma}(\vct{\lambda})\;:=\;\|\vct{\sigma}\|_2\,\sqrt{\vct{\lambda}},
\]
so that $\mathcal N(\vct{\sigma}(\vct{\lambda}))=\vct{\lambda}$. For each $i\ge 1$, we then define the corresponding
hard instance
\[
D_i(\vct{\lambda})\;:=\;D^{(i)}\!\bigl(\vct{\sigma}(\vct{\lambda})\bigr)\in\mathcal H^d.
\]

The next lemma formalizes the resulting correspondence: decisions in $\Delta_\rho(\vct{\lambda}^\star)$ index a local
family of hard instances, and restricting to this family can only strengthen a lower bound.

\begin{lemma}[Decision-indexed hard instances]\label{lem:identification}
For each $i\ge 1$ and each $\vct{\lambda}\in\Delta_\rho(\vct{\lambda}^\star)$,
\[
\mathcal N\!\bigl(\vct{\sigma}(D_i(\vct{\lambda}))\bigr)=\vct{\lambda}.
\]
Moreover, for any policy $\pi$,
\[
\sup_{\vct D\in\mathcal H_{\rho}(\vct{\sigma})}\Reg(\pi,\vct D)
\;\ge\;
\limsup_{i\to\infty}\ \sup_{\vct{\lambda}\in\Delta_\rho(\vct{\lambda}^\star)}\Reg\!\bigl(\pi, D_i(\vct{\lambda})\bigr),
\]
and the same inequality holds after taking $\inf_\pi$. In particular,
\begin{equation*}
    \mathcal V_{\rho,\tau}(\vct{\sigma}) \geq \inf_{\pi \in \Pi^{\mathrm{sym}}_{\rho, \tau}(\vct{\sigma})}\limsup_{i\to\infty}\ \sup_{\vct{\lambda}\in\Delta_\rho(\vct{\lambda}^\star)}\Reg\!\bigl(\pi, D_i(\vct{\lambda})\bigr).
\end{equation*}
\end{lemma}

Lemma~\ref{lem:identification} shows that the search for a worst-case instance reduces to a search over a local
set of decisions $\Delta_\rho(\vct{\lambda}^\star)$, with statistical indistinguishability controlled armwise by
$\mathrm{VLC}_\rho(\cdot\mid\mathcal H)$ through the selectors \eqref{eq:optimalrepresentation-app}.

\subsection{Decision Mismatch}\label{app:mismatch}

Next, we rewrite the normalized regret into a form that highlights the discrepancy between a policy's terminal
decision and the target decision induced by the instance. Fix $i\ge 1$ and a decision
$\vct{\lambda}\in\Delta_\rho(\vct{\lambda}^\star)$. By construction, the decision-indexed instance
$D_i(\vct{\lambda})$ has decision-relevant profile $\mathcal N(\vct{\sigma}(D_i(\vct{\lambda})))=\vct{\lambda}$.
When a policy $\pi$ is run on this instance, it produces terminal counts $\vct n=(n_1,\dots,n_d)$ and hence a
random terminal decision
\[
\bar{\vct{\lambda}}\;:=\;\frac{\vct n}{T}\in\Delta_d.
\]
The regret can therefore be viewed as the expected cost of implementing $\bar{\vct{\lambda}}$ when the target
decision is $\vct{\lambda}$.

We formalize this by introducing the \emph{decision mismatch} function $r(\cdot\|\cdot)$, which quantifies the
penalty for acting according to one decision when another is the target. In our setting,
\begin{equation}\label{eq:def-r-app}
r(\bar{\vct{\lambda}}\|\vct{\lambda})
\;:=\;
\left\|\frac{\vct{\lambda}}{\bar{\vct{\lambda}}}\right\|_\infty - 1
\;=\;
\max_{k\in[d]}\frac{\lambda_k}{\bar\lambda_k}-1,
\end{equation}
with the convention $r(\bar{\vct{\lambda}}\|\vct{\lambda})=+\infty$ if $\bar\lambda_k=0$ for some $k$.
A direct calculation from the definition of normalized regret shows that, for each $i$ and $\vct{\lambda}$,
\begin{equation}\label{eq:regret-as-mismatch}
\Reg\!\bigl(\pi, D_i(\vct{\lambda})\bigr)
\;=\;
\E\!\left[r(\bar{\vct{\lambda}}\|\vct{\lambda})\right],
\end{equation}
where the expectation is over the randomness of the policy and the samples through the random terminal decision
$\bar{\vct{\lambda}}$.

The advantage of \eqref{eq:regret-as-mismatch} is that it decouples analytic and statistical components. The
non-additive structure of the objective is entirely encapsulated by the deterministic function $r(\cdot\|\cdot)$,
while the behavior of any policy $\pi$ is reduced to the distributional properties of a single random vector, the
terminal decision $\bar{\vct{\lambda}}$. This highlights a fundamental departure from classic bandit frameworks:
in those settings, regret is typically a cumulative sum of separable per-round losses, whereas here the objective
is evaluated holistically on the terminal allocation. Consequently, the relevant geometry is induced by $r$ on the
decision set, rather than by additive per-arm contributions.

Finally, while the notation $r(\cdot\|\cdot)$ is suggestive of a divergence, it is simply a problem-specific cost
function derived from the regret definition. It satisfies the basic separation property
$r(\vct{\lambda}\|\vct{\lambda})=0 \implies \vct{\lambda} = \bar{\vct{\lambda}}$ and $r(\bar{\vct{\lambda}}\|\vct{\lambda})>0$ for $\bar{\vct{\lambda}}\neq\vct{\lambda}$,
and increases as $\bar{\vct{\lambda}}$ deviates from the target $\vct{\lambda}$.

\subsection{Structured Adversarial Construction}\label{app:construction}

The analysis now requires simplifying the supremum over the local decision set $\Delta_\rho(\vct{\lambda}^\star)$
to make the lower bound tractable. The core strategy is to replace this continuous set with a finite, structured
collection of hardest-to-distinguish alternatives. This reduction is principled: shrinking the set of possible
instances into a finite set makes the problem easier for the policy (in fact, it makes the problem as easy as classification on a specific loss function), hence any lower bound derived on the restricted family remains
valid for the original problem. The resulting task is a discrete multi-class testing problem.

\paragraph{A finite ``sign-cube'' of directions.}
The central step is to replace the continuum of deviation directions by a finite set that preserves the essential
combinatorial difficulty. We restrict to a finite design
\[
\mathcal X \subset \{-1,+1\}^d \cap \one^\perp,
\]
so that instead of estimating a continuous decision, the policy must identify which label $x\in\mathcal X$ was
chosen by the adversary. 

\paragraph{Embedding the code into the local decision set.}
With the combinatorial structure fixed by $\mathcal X$, we embed these discrete directions into the local geometry
around the base decision $\vct{\lambda}^\star\in\Delta_d$. A raw direction $x\in\mathcal X$ must be mapped to a
valid nearby decision. We do so by choosing a linear map $\vct A\in\R^{d\times d}$ and defining
\begin{equation}\label{eq:lambda-param-app}
\vct{\lambda}_x
\;:=\;
\vct{\lambda}^\star\odot\Bigl(\one+\frac{\rho}{\sqrt d}\,\vct{A}x\Bigr),
\qquad x\in\mathcal X,
\end{equation}
where $\odot$ denotes the Hadamard product. This parameterization cleanly separates (i) the discrete directions
$x\in\mathcal X$ from (ii) the orientation/anisotropy encoded by $\vct A$.

For \eqref{eq:lambda-param-app} to be valid, we must ensure
$\vct{\lambda}_x\in\Delta_\rho(\vct{\lambda}^\star)$. This is equivalent to
\begin{equation}\label{eq:scaling_validity-app}
\vct{A}\mathcal{X} \perp (\vct{\lambda}^\star), \text{ }\text{and}\quad \Bigl\|\vct A \mathcal{X}\Bigr\|_\infty \le \sqrt d.
\end{equation}
 
\begin{lemma}[Exact feasibility conditions for the local simplex]\label{lem:feasibility_app}
Fix $\vct\lambda^\star\in\Delta_d$ and $\rho\in(0,1)$. For each $x\in\mathcal X$, define
\[
\vct\lambda_x \;:=\; \vct\lambda^\star\odot\Bigl(\one+\frac{\rho}{\sqrt d}\,\vct A x\Bigr).
\]
Then the following are equivalent:
\begin{enumerate}[label=\textup{(\roman*)}, leftmargin=*, itemsep=0.15em]
\item $\vct\lambda_x\in\Delta_\rho(\vct\lambda^\star)$, i.e.
$\bigl\|\vct\lambda_x/\vct\lambda^\star-\one\bigr\|_2\le \rho$ and $\vct\lambda_x\in\Delta_d$.
\item The vector $\vct u_x:=\vct A x$ satisfies the two constraints
\[
\langle \vct\lambda^\star,\vct u_x\rangle=0
\qquad\text{and}\qquad
\|\vct u_x\|_2\le \sqrt d .
\]
\end{enumerate}
In particular, if $\vct A\mathcal X\subset (\vct\lambda^\star)^\perp$ and $\sup_{x\in\mathcal X}\|\vct A x\|_2\le \sqrt d$,
then $\{\vct\lambda_x\}_{x\in\mathcal X}\subset \Delta_\rho(\vct\lambda^\star)$.
\end{lemma}

\paragraph{From decisions to hard instances.}
Using the decision-indexed generator from Section~\ref{app:framework}, each candidate decision $\vct{\lambda}_x$
induces a hard instance
\[
\vct D_i(x)\;:=\;D_i(\vct{\lambda}_x)\in\mathcal H^d,
\qquad x\in\mathcal X.
\]
Thus, the adversary randomizes over the finite family $\{\vct D_i(x):x\in\mathcal X\}$.

\paragraph{Separation: regret reduces to decision mismatch on the constructed family.}
The key point is that, on this family, normalized regret is exactly the expected mismatch between the random
terminal decision $\bar{\vct{\lambda}}=\vct n/T$ and the target decision $\vct{\lambda}_x$.

\begin{lemma}[Separation on the constructed family]\label{lem:separation_app}
For every policy $\pi$, every $i\ge 1$, and every $x\in\mathcal X$,
\[
\Reg(\pi,\vct D_i(x))
\;=\;
\E_{\bar{\vct \lambda}\sim(\pi,\vct D_i(x))}\!\left[
r\!\left(\bar{\vct \lambda}\ \Big\|\ \vct{\lambda}_x\right)\right],
\]
where $r(\cdot\|\cdot)$ is the decision mismatch defined in \eqref{eq:def-r-app}.
\end{lemma}

Together, Lemmas~\ref{lem:gv_design_tuning}--\ref{lem:separation_app} reduce the lower bound problem to analyzing a
finite multi-class testing problem over $x\in\mathcal X$, with statistical indistinguishability controlled by
$\mathrm{VLC}_\rho$ through the instance generator $\vct{A}$ and decision separation controlled by the loss-induced geometry $r$.

\subsection{Reduction to a classification problem}\label{app:reduction}

A direct lower-bound analysis of the expected mismatch cost
$\E\big[r(\bar{\vct{\lambda}}\|\vct{\lambda})\big]$ is difficult, since it requires characterizing the distribution
of the policy's random terminal decision $\bar{\vct{\lambda}}=\vct n/T$ for an arbitrary adaptive policy. To bridge
this continuous--discrete gap, we follow a standard information-theoretic template: we discretize the local decision
set by a finite packing and reduce the problem to a multi-class testing task. The novelty in our setting is that,
because the loss is non-additive and non-Euclidean, we must first identify a loss-compatible geometry that turns
misclassification into a quantitative regret penalty.

\subsubsection*{The minimum-mismatch decoder}

Fix a finite set of candidate decisions $\{\vct{\lambda}_x\}_{x\in\mathcal X}\subset\Delta_\rho(\vct\lambda^\star)$ (as constructed in
Section~\ref{app:construction}). Given an output $\bar{\vct{\lambda}}\in\Delta_\rho(\vct\lambda^\star)$, we decode a label by nearest-neighbor
in the mismatch loss $r$:
\begin{definition}[Minimum-mismatch decoder]\label{def:decoder-app}
Given $\bar{\vct{\lambda}}\in\Delta_\rho(\vct\lambda^\star)$, define
\begin{equation}\label{eq:decoder_rule_app}
\bar{x}\ \in\ \operatorname*{argmin}_{x'\in\mathcal X}\ r\!\left(\bar{\vct{\lambda}}\,\big\|\,\vct{\lambda}_{x'}\right).
\end{equation}
\end{definition}
This decoder is canonical for our objective: it projects the continuous terminal decision onto the discrete packing
using exactly the loss that defines regret.

\subsubsection*{A loss-induced geometry on decisions}

To lower bound regret by a term that depends only on the discrete event $\{\bar x\neq x\}$, we need a function that
quantifies the intrinsic cost of confusing two labels. We define this cost through a minimax ambiguity principle.

\begin{definition}[Symmetric ambiguity cost]\label{def:symm_dissimilarity_app}
For $\vct{\lambda}_a,\vct{\lambda}_b\in\Delta_d$, define
\[
\tilde r(\vct{\lambda}_a,\vct{\lambda}_b)
\;:=\;
\inf_{\vct{\lambda}\in\Delta_d}\ \max\!\left\{
r(\vct{\lambda}\|\vct{\lambda}_a),\ r(\vct{\lambda}\|\vct{\lambda}_b)
\right\}.
\]
\end{definition}
By construction, $\tilde r(\vct{\lambda}_a,\vct{\lambda}_b)$ is the smallest worst-case mismatch incurred by a single
decision when the target is either $\vct{\lambda}_a$ or $\vct{\lambda}_b$. This is the loss-induced notion of
separation we use throughout the lower bound.

\begin{lemma}[Projection reduces regret]\label{lem:projection-reduces-regret-app}
Fix $\{\vct{\lambda}_x\}_{x\in\mathcal X}\subset\Delta_\rho(\vct\lambda^\star)$ and let $x\in\mathcal X$ be the true label. For any policy
and the decoder $\bar x$ from Definition~\ref{def:decoder-app},
\[
\E\!\left[r\!\left(\bar{\vct{\lambda}}\,\big\|\,\vct{\lambda}_x\right)\right]
\;\ge\;
\E\!\left[\tilde r\!\left(\vct{\lambda}_x,\vct{\lambda}_{\bar x}\right)\right]
\;\ge\;
\Pr(\bar x\neq x)\cdot \min_{b\neq x}\tilde r(\vct{\lambda}_x,\vct{\lambda}_b).
\]
\end{lemma}

The first inequality converts regret into an ambiguity penalty, and the second isolates the misclassification
probability. Thus, the lower bound reduces to controlling two factors: a statistical classification error and a
geometric separation term.

\subsubsection*{A tractable $\ell_1$ separation}

A key simplification in our setting is that the geometry induced by $\tilde r$ is equivalent to the
$\ell_1$ geometry on $\Delta_d$.

\begin{lemma}[$\ell_1$ lower bound for $\tilde r$]\label{lem:rtilde_l1_app}
For all $\vct{\lambda}_a,\vct{\lambda}_b\in\Delta_d$,
\[
\tilde r(\vct{\lambda}_a,\vct{\lambda}_b)\ = \frac12\|\vct{\lambda}_a-\vct{\lambda}_b\|_1.
\]
\end{lemma}

In particular, if the finite family $\{\vct{\lambda}_x\}$ is well separated in $\ell_1$, then any misclassification
event $\{\bar x\neq x\}$ forces a nontrivial regret penalty. In the next subsections, we bound the classification
error via Fano (using structured KL control from $\mathrm{VLC}_\rho$) and bound the decision separation via the
explicit parameterization of $\vct{\lambda}_x$.

\subsection{Bounding the classification error}\label{app:classification_error}

The reduction in Lemma~\ref{lem:projection-reduces-regret-app} shows that the regret lower bound factors through the
misclassification probability $\Pr(\bar x\neq x)$. This is a multi-class testing problem: the adversary draws
a label $x\sim\mathcal U(\mathcal X)$, the policy observes the bandit transcript $\mathcal F_T$, and outputs a
decoded label $\bar x=\bar x(\mathcal F_T)$. The ability of any policy to succeed is limited by the amount of
information the transcript carries about $x$. A standard information-theoretic tool is Fano's inequality, which
lower bounds $\Pr(\bar x\neq x)$ when the mutual information $I(x;\mathcal F_T)$ is small.

For bandit models, $I(x;\mathcal F_T)$ can be upper bounded by average pairwise transcript KL, and transcript KL
decomposes into a sum over arms of expected pulls times arm-level KL. Our local construction provides control of
both terms. First, by construction of the VLC selectors (Step~2 in Section~\ref{app:framework}), the arm-level KL
between two labels is proportional to the squared decision perturbation, with proportionality constant
$\mathrm{VLC}_\rho(\sigma_{k}\mid\mathcal H)$ (up to a vanishing $o(1)$ term along $i\to\infty$). Second, for
policies in the local minimax class, the tolerance condition forces sampling proportions to remain in the local
regime, yielding the pull-count control $\E[n_k]\approx T\,\lambda_{k}^\star$ (up to a vanishing error as
$(\rho,\tau)\to(0,0)$). Combining these two controls yields a geometric upper bound on mutual information in terms
of a Frobenius norm, stated next.

\begin{lemma}[Fano bound with structured KL]\label{lem:fano_bound_app}
Let $\mathcal X\subset\{\pm1\}^d\cap\one^\perp$ be nonempty and let $x\sim\mathcal U(\mathcal X)$.
Let $\{\vct\lambda_x\}_{x\in\mathcal X}$ be defined by \eqref{eq:lambda-param-app} for some $\vct A\in\R^{d\times d}$,
and set
\[
\vct{\Sigma}:=\Cov\big(\mathcal U(\mathcal X)\big),
\qquad
\vct\Gamma:=\mathrm{diag}\!\Big(\lambda_{k}^\star\,\mathrm{VLC}_{\rho}(\sigma_{k}\mid\mathcal H)\Big)_{k=1}^d,
\]
where $\vct{\lambda}^\star=\mathcal N(\vct{\sigma})$ is the base decision. There exists a function
$\varepsilon(\rho,\tau)$ with $\varepsilon(\rho,\tau)\to 0$ as $(\rho,\tau)\to(0,0)$ such that for any policy $\pi\in\Pi^{\mathrm{sym}}_{\rho,\tau}(\vct{\sigma})$,
\[
\limsup_{i\to\infty}\ \Pr(\bar x\neq x)
\ \ge\
1-\frac{\frac{2T\rho^2}{d}\,\big\|\sqrt{\vct{\Gamma}}\,\vct{A}\,\sqrt{\vct{\Sigma}}\big\|_F^2\,
\bigl(1+\varepsilon(\rho,\tau)\bigr)+\log 2}{\log|\mathcal X|}.
\]
In particular, the error term does not depend on the choice of the policy $\pi \in \Pi^{\mathrm{sym}}_{\rho,\tau}(\vct{\sigma})$.
\end{lemma}

Lemma~\ref{lem:fano_bound_app} isolates the statistical difficulty in a single geometric quantity:
$\|\sqrt{\vct\Gamma}\,\vct A\,\sqrt{\vct\Sigma}\|_F^2$, where $\|\cdot\|_F$ is the Frobenius norm. The diagonal matrix $\vct\Gamma$ encodes armwise information
curvatures through $\mathrm{VLC}_\rho(\sigma_{k}\mid\mathcal H)$, while $\vct\Sigma$ encodes the geometry of the
design $\mathcal X$. The embedding map $\vct A$ transports the design into the decision neighborhood, and the
Frobenius term captures their interaction at the level of mutual information.
\subsection{Bounding the decision cost}\label{app:decision_cost}

Having bounded the classification error, we now lower bound the second component in the regret reduction: the
conditional decision cost $\E[\tilde r(\vct{\lambda}_x,\vct{\lambda}_{\bar x})\mid \bar x\neq x]$. The key point
is that our hard family is indexed by hypercube labels $x\in\mathcal X\subset\{\pm1\}^d$, so the disagreement
structure of $x-\bar x$ has a simple combinatorial form. Define the disagreement projector
\[
\vct P_{\neq}\ :=\ \mathrm{diag}\!\big(\mathbf 1\{x_k\neq \bar x_k\}\big),
\qquad\text{so that}\qquad
x-\bar x \;=\; 2\,\vct P_{\neq}x .
\]
This representation allows us to connect decision separation directly to the $\ell_1$ geometry of the loss.

\begin{lemma}[Lower bound on conditional decision cost]\label{lem:decision_cost_app}
Assume the candidates are parameterized as
\[
\vct{\lambda}_x
=
\vct{\lambda}^\star\odot\Bigl(\one+\frac{\rho}{\sqrt d}\,\vct{A}x\Bigr),
\qquad x\in\mathcal X\subset\{\pm1\}^d,
\]
and define $\vct\Lambda:=\mathrm{diag}(\vct{\lambda}^\star)$. Then on the event $\{\bar x\neq x\}$,
\[
\tilde r(\vct{\lambda}_x,\vct{\lambda}_{\bar x})
=
\frac{\rho}{\sqrt d}\,\big\|\vct{\Lambda}\,\vct{A}\,(\vct P_{\neq}x)\big\|_1,
\]
and hence
\[
\E\!\left[\tilde r(\vct{\lambda}_x,\vct{\lambda}_{\bar x})\mid \bar x\neq x\right]
\ =
\frac{\rho}{\sqrt d}\,
\E\!\left[\big\|\vct{\Lambda}\,\vct{A}\,(\vct P_{\neq}x)\big\|_1\ \Big|\ \bar x\neq x\right].
\]
\end{lemma}

The equality in Lemma~\ref{lem:decision_cost_app} consists of a conventional scale factor $\rho/\sqrt d$ and a
core structural term $\|\vct\Lambda\,\vct A\,\vct P_{\neq}x\|_1$. This term quantifies the magnitude of the
decision penalty by composing the policy-dependent disagreement vector $\vct P_{\neq}x$ with the embedding
$\vct A$ and the baseline scaling $\vct\Lambda$. It is precisely this interaction that creates the adversary's
tradeoff: the same embedding $\vct A$ that controls statistical indistinguishability through the Frobenius term
in Lemma~\ref{lem:fano_bound_app} also controls the magnitude of the penalty incurred when the policy makes an
error.

\subsection{Tuning the design}\label{app:tuning}

The structured construction reduces the lower bound to a discrete testing problem indexed by $(\mathcal X,\vct A)$.
We now explain how these objects are chosen so as to (i) make the classification problem hard and (ii) ensure that
any error is costly, while maintaining feasibility of the local parameterization.

\paragraph{Design of $\mathcal X$.}
The set $\mathcal X$ plays a purely combinatorial role: it defines the discrete set of candidate directions and
hence the number of labels the policy must distinguish. To make the testing problem intrinsically hard, we want
$|\mathcal X|$ to be exponentially large in $d$. At the same time, to amplify the cost of any error, we want
$\mathcal X$ to be well-separated in Hamming distance so that any misclassification produces a large disagreement
pattern (and hence a large projector $\vct P_{\neq}$). These two requirements are met by a standard
Gilbert--Varshamov construction. Such a design only exists in even dimensions $d \in 2 \cdot \mathbb{N}$. For simplicity, we assume that $d$ is even. Proving the lower bound on odd dimensions can be easily extended by dropping one arm or padding with a dummy coordinate.

\begin{lemma}[Hypercube design]\label{lem:gv_design_tuning}
There exists a universal constant $\kappa>0$ and a set $\mathcal X\subset\{-1,1\}^d\cap\one^\perp$ such that
$\min_{x\neq x'} d_H(x,x')\ge d/4$ and $\log|\mathcal X|\ge \kappa d$.
\end{lemma}

For the remainder of the analysis we fix such a design $\mathcal X$.

\paragraph{Design of the representation map $\vct A$.}
Once $\mathcal X$ is fixed, the remaining degree of freedom is the embedding $\vct A$, which controls the
\emph{orientation} of the packing inside the local decision neighborhood through
$\vct\lambda_x=\vct\lambda^\star\odot(\one+\frac{\rho}{\sqrt d}\vct A x)$. The key point is that the same
$\vct A$ governs both sides of the lower bound:
\begin{itemize}[leftmargin=*, itemsep=0.2em]
\item \emph{Distinguishability (classification).} The mutual information bound from Fano yields a term of the form
$\|\sqrt{\vct \Gamma}\,\vct A\,\sqrt{\vct\Sigma}\|_F^2$, where $\vct\Sigma=\Cov(\mathcal U(\mathcal X))$ and
$\vct \Gamma=\mathrm{diag}(\lambda_{k}^\star\,\mathrm{VLC}_\rho(\sigma_{k}\mid\mathcal H))$.
\item \emph{Decision cost.} On the event $\{\bar x\neq x\}$, the $\ell_1$ geometry yields a term of the form
$\|\vct\Lambda\,\vct A(\vct P_{\neq}x)\|_1$, where $\vct\Lambda=\mathrm{diag}(\vct\lambda^\star)$.
\end{itemize}
Thus, choosing $\vct A$ amounts to choosing an embedding that is simultaneously hard to distinguish and costly to
confuse. We select $\vct A$ by a deterministic derandomization argument; let us explain the motivation.

At this point the lower bound depends on $\vct A$ through two competing geometric quantities (one governing
distinguishability, one governing the cost of confusion), and optimizing this tradeoff directly over all matrices
is analytically hard. Random matrix theory provides a mature calculus for computing expectations of spectral and
norm functionals under rotations, so averaging over a symmetric matrix ensemble often yields clean first-order
expressions for the quantities we need.

We separate anisotropy from orientation by writing a right polar-type decomposition
\[
\vct A \;=\; \vct S \vct O,
\]
where $\vct S$ is a deterministic positive semi-definite matrix that encodes the anisotropy we wish to impose, and $\vct O$ is an
orthogonal matrix that randomizes orientation. 

The anisotropy matrix $\vct S$ fixes the local shape of the
construction and prepares the problem for a uniform orthogonal randomization: since $\vct{S}$ is supposed to deal with any anisotropies,
averaging uniformly over $\vct O$ is the natural way to remove orientation effects. Concretely, we do not average over the whole orthogonal group: to maintain feasibility we restrict the orientation
to orthogonal matrices that send the (normalized) base direction $\vct\lambda$ to the all-ones direction $\one$.
Since orthogonal maps preserve Euclidean norms, we impose this constraint at the level of unit vectors and sample
$\vct O$ from the compact set
\[
\mathsf{O}(\vct\lambda\!\to\!\one)
\;:=\;
\Bigl\{\vct O\in\mathsf{O}(d):\ \vct O\,\frac{\vct\lambda}{\|\vct\lambda\|_2}
=
\frac{\one}{\sqrt d}\Bigr\}.
\]
We take $\vct O$ to be \emph{uniform} over $\mathsf{O}(\vct\lambda\!\to\!\one)$ in the canonical way induced by Haar
measure: fix any deterministic $\vct Q\in\mathsf{O}(d)$ such that
$\vct Q\,\frac{\vct\lambda}{\|\vct\lambda\|_2}=\frac{\one}{\sqrt d}$, draw $\vct U$ Haar-uniform from the stabilizer
subgroup $\mathrm{Stab}(\one):=\{\vct U\in\mathsf{O}(d):\vct U\one=\one\}$, and set $\vct O=\vct U\vct Q$.
This law is supported on $\mathsf{O}(\vct\lambda\!\to\!\one)$, is invariant under the natural symmetries of the
constraint, and does not depend on the particular choice of $\vct Q$. Since $\mathsf{O}(\vct\lambda\!\to\!\one)$ is
compact, the resulting distribution is a bona fide probability measure and the expectations we use are
well-defined; we therefore omit the standard measure-theoretic formalities.

After this symmetrization step, the
analysis depends on $\vct S$ only through rotation-invariant quantities. We then derandomize: once we compute the
average tradeoff over $\vct O$, there must exist a fixed orthogonal $\vct O$ (and hence a fixed $\vct A=\vct S\vct O$)
that achieves at least the same tradeoff. This reduces the hard instance design to choosing the deterministic
anisotropy $\vct S$ (optimized later), while the orthogonal factor $\vct O$ serves as the analytic device that
makes the calculation tractable.

\begin{lemma}[Representation Design]\label{lem:good_representation_tuning}
Let $\vct\Sigma=\Cov(\mathcal U(\mathcal X))$, $\vct\Lambda=\mathrm{diag}(\vct\lambda^\star)$, and
\[
\vct\Gamma
:=
\mathrm{diag}\!\Big(\lambda_k^\star\,\mathrm{VLC}_\rho(\sigma_k\mid\mathcal H)\Big)_{k=1}^d.
\]
There exist $\vct A_\star\in\R^{d\times d}$ such that
\begin{equation*}
    \frac{
\E_{x \sim \mathcal{U}(\mathcal{X}), \bar{x}}\!\left[\big\|\vct\Lambda\vct A_\star(\vct P_{\neq}x)\big\|_1\ \middle|\ \bar x\neq x\right]
}{
\big\|\sqrt{\vct\Gamma} \vct A_\star\,\sqrt{\vct\Sigma}\big\|_F
}
\ \gtrsim \sqrt{\sum_{k=1}^d
\frac{\hts_k(\vct \lambda^\star)}
{\mathrm{VLC}_\rho(\sigma_k\mid\mathcal H)}},
\end{equation*}
where $\gtrsim$ hides universal constants and a possible error term that is only a function of $d$. Moreover, the inequality above is scale invariant for $\vct{A}_\star$, i.e., it remains valid for any $t \cdot \vct{A}_\star$ for $t > 0$.
\end{lemma}

\paragraph{Feasibility and resolution.}
Finally, the packing must remain inside the local neighborhood. Feasibility requires that the multiplicative
perturbations in $\vct\lambda_x=\vct\lambda^\star\odot(\one+\frac{\rho}{\sqrt d}\vct A x)$ preserve positivity and
normalization. This imposes a resolution constraint relating $\rho$ and $(d,T)$.

\begin{lemma}[Feasibility and resolution]\label{lem:resolution_tuning}
Setting $\rho^* = \Theta\left(\sqrt{\frac{d}T{}}\right)$ implies that for all $\rho \gtrsim \rho^*$ $\vct{A}_\star \mathcal{X} \perp \vct \lambda^\star$ and $\|\vct{A}_\star\mathcal{X}\|_2 \leq \sqrt{d}$.
\end{lemma}

Together, Lemmas~\ref{lem:gv_design_tuning}--\ref{lem:resolution_tuning} complete the tuning of the adversarial
construction: $\mathcal X$ provides a large, well-separated label set; $\vct A$ provides a favorable embedding
capturing both orientation and anisotropy; and $\rho$ must exceed the budget-imposed resolution $\rho^\star = \frac{1}{2}\sqrt{d/T}$.

\subsection{Putting everything together: Proof of Theorem \ref{thm:main}}\label{app:proof_of_main_theorem}

We now synthesize the preceding components into a proof of Theorem~\ref{thm:main}. Throughout, fix a base profile
$\vct\sigma\in\R_+^d$, a radius $\rho\in(0,1)$, and a tolerance $\tau>0$ such that the local minimax class
$\Pi^{\mathrm{sym}}_{\rho,\tau}(\vct\sigma)$ is nonempty. Let $\vct\lambda^\star:=\mathcal N(\vct\sigma)$ and
$\vct\Lambda:=\mathrm{diag}(\vct\lambda^\star)$.

Fix $\rho\ge \rho^\star:=\tfrac12\sqrt{d/T}$. Let $\mathcal X\subset\{-1,1\}^d\cap\one^\perp$ be a design with
$\log|\mathcal X|\ge \kappa d$ and $\min_{x\neq x'} d_H(x,x')\ge d/4$ (which existence is guaranteed via Lemma~\ref{lem:gv_design_tuning}). For a
linear map $\vct A\in\R^{d\times d}$ (chosen below), define candidate decisions
\[
\vct\lambda_x
=
\vct\lambda^\star\odot\Bigl(\one+\frac{\rho}{\sqrt d}\,\vct A x\Bigr),
\qquad x\in\mathcal X,
\]
and let $\vct D_i(x):=D_i(\vct\lambda_x)$ be the corresponding decision-indexed instances constructed from the
$\mathrm{VLC}_\rho$ selectors. For any policy $\pi\in\Pi^{\mathrm{sym}}_{\rho,\tau}(\vct\sigma)$, let
$\bar{\vct\lambda}=\vct n/T$ be its terminal decision and let $\bar x$ be the minimum-mismatch decoder
(see Definition~\ref{def:decoder-app}).

\paragraph{Step 1: reduce regret to classification error and decision cost.}
By Lemma~\ref{lem:separation_app}, for each $x\in\mathcal X$,
\[
\Reg(\pi,\vct D_i(x))
=
\E\!\left[r(\bar{\vct\lambda}\|\vct\lambda_x)\right].
\]
Applying Lemma~\ref{lem:projection-reduces-regret-app} yields
\begin{equation}\label{eq:reduce-product}
\Reg(\pi,\vct D_i(x))
= \E\!\left[r(\bar{\vct\lambda}\|\vct\lambda_x)\right]
\ \ge\
\Pr(\bar x\neq x)\cdot \E\!\left[\tilde r(\vct\lambda_x,\vct\lambda_{\bar x})\mid \bar x\neq x\right].
\end{equation}

\paragraph{Step 2: bound the classification error by Fano.}
Let $x_{\star}\sim\mathcal U(\mathcal X)$ be uniform on $\mathcal X$. Lemma~\ref{lem:fano_bound_app} gives
\begin{equation}\label{eq:fano-main}
\limsup_{i\to\infty}\ \Pr(\bar x\neq x_{\star})
\ \ge\
1-\frac{\frac{2T\rho^2}{d}\,\big\|\sqrt{\vct\Gamma}\,\vct A\,\sqrt{\vct\Sigma}\big\|_F^2\,(1+\varepsilon(\rho,\tau))+\log 2}{\log|\mathcal X|},
\end{equation}
where $\vct\Sigma=\Cov(\mathcal U(\mathcal X))$ and $\vct\Gamma :=\mathrm{diag}\!\Big(\lambda_k^\star\,\mathrm{VLC}_\rho(\sigma_k\mid\mathcal H)\Big)_{k=1}^d.
$
Since $\log|\mathcal X|\ge \kappa d$, it suffices to choose the resolution $\rho > 0$ so that the fraction in
\eqref{eq:fano-main} is bounded away from $1$. This is possible if we choose
\begin{equation}\label{eq:tune-A-fano}
\frac{2T\rho^2}{d \log |\mathcal{X}|}\,\big\|\sqrt{\vct\Gamma}\,\vct A\,\sqrt{\vct\Sigma}\big\|_F^2\ \geq \frac{1}{2} \iff \rho \geq \frac{1}{2}\sqrt{\frac{d \log |\mathcal{X}|}{T\big\|\sqrt{\vct\Gamma}\,\vct A\,\sqrt{\vct\Sigma}\big\|_F^2}}.
\end{equation}
Then, we have
\[
\Pr(\bar x\neq x_\star)\ \ge\ 1 - \frac{1 + \epsilon(\rho, \tau)}{2} - \frac{\log 2/\kappa}{ d} = \Omega(1 - \epsilon(\rho, \tau)).
\]

\paragraph{Step 3: bound the conditional decision cost.}
On $\{\bar x\neq x\}$, Lemma~\ref{lem:decision_cost_app} yields
\begin{equation}\label{eq:cost-main}
\E\!\left[\tilde r(\vct\lambda_x,\vct\lambda_{\bar x})\mid \bar x\neq x\right]
\ =
\frac{\rho}{\sqrt d}\,
\E\!\left[\big\|\vct\Lambda\,\vct A\,(\vct P_{\neq}x)\big\|_1\ \Big|\ \bar x\neq x\right] = \frac{1}{2}\sqrt{\frac{\log |\mathcal{X}|}{T}}\frac{\E\!\left[\big\|\vct\Lambda\,\vct A\,(\vct P_{\neq}x)\big\|_1\ \Big|\ \bar x\neq x\right]}{\big\|\sqrt{\vct\Gamma}\,\vct A\,\sqrt{\vct\Sigma}\big\|_F},
\end{equation}
where $\vct P_{\neq}=\mathrm{diag}(\mathbf 1\{x_k\neq \bar x_k\})$.

\paragraph{Step 4: choose a good orientation and conclude.}
The remaining task is to choose an embedding $\vct A$ that makes the ratio between the decision-cost term in
\eqref{eq:cost-main} and the distinguishability term in \eqref{eq:tune-A-fano} large. This is achieved by the
representation lemma (Lemma~\ref{lem:good_representation_tuning}), which provides a choice of $\vct A$ such that, up to universal constants and vanishing terms,
\begin{equation}\label{eq:ratio-main}
    \frac{
\E\!\left[\big\|\vct\Lambda\vct A_\star(\vct P_{\neq}x)\big\|_1\ \middle|\ \bar x\neq x\right]
}{
\big\|\sqrt{\vct\Gamma}\vct A_\star\,\sqrt{\vct\Sigma}\big\|_F
}
\ \gtrsim \sqrt{\sum_{k=1}^d
\frac{\lambda_k^\star\left(1-\left(\lambda_k^\star/\|\vct\lambda^\star\|_2\right)^2\right)}
{\mathrm{VLC}_\rho(\sigma_k\mid\mathcal H)}} = \sqrt{\sum_{k=1}^d
\frac{\hts_k(\vct \lambda^\star)}
{\mathrm{VLC}_\rho(\sigma_k\mid\mathcal H)}},
\end{equation}
Combining
\eqref{eq:reduce-product}, \eqref{eq:cost-main}, \eqref{eq:ratio-main},  the constant lower bound on
$\Pr(\bar x\neq x)$, as well as $\log |\mathcal{X}| \geq \kappa d$ yields
\[
\limsup_{i\to\infty}\ \sup_{x\in\mathcal X}\Reg(\pi,\vct D_i(x))
\ \gtrsim\
(1-\epsilon(\rho,\tau))\sqrt{\frac{d}{T}}\sqrt{\sum_{k=1}^d
\frac{\hts_k(\vct \lambda^\star)}
{\mathrm{VLC}_\rho(\sigma_k\mid\mathcal H)}},
\]
where we implicitly used that the constructed instance is feasible by Lemma \ref{lem:feasibility_app}. By recalling that
$\vct \lambda^\star \propto \vct \sigma$ and that $\hts$ is scale-invariant, and taking $\inf_{\pi\in\Pi^{\mathrm{sym}}_{\rho,\tau}(\vct\sigma)}$ (which is possible since the error term $\epsilon$ does not depend on the choice of $\pi$, by Lemma \ref{lem:fano_bound_app}) and using Lemma~\ref{lem:identification}, we complete the proof for Theorem \ref{thm:main}.

\section{Proof of Technical Lemmas}\label{sec:proof-aux}
\begin{proof}{\textit{Proof of Lemma~\ref{lem:VLC-Fisher}.}}
Fix $\sigma\in\Rpos$ and write $\theta:=\sigma^2$. Recall
\[
\mathrm{VLC}_{\rho}(\sigma\mid\mathcal H)
=
\inf_{D\in\mathfrak S_\rho(\sigma)}\ 
\sup_{v_a,v_b\in B_\rho(\sigma)}
\frac{\KL(D(v_a)\,\|\,D(v_b))}
{\Bigl(\frac{v_a^2-v_b^2}{\sigma^2}\Bigr)^2},
\qquad
B_\rho(\sigma)=\Bigl\{v\in\Rpos:\Bigl|\frac{v^2-\sigma^2}{\sigma^2}\Bigr|\le\rho\Bigr\}.
\]
Let $\theta_a:=v_a^2$ and $\theta_b:=v_b^2$. Then $v_a,v_b\in B_\rho(\sigma)$ implies
\[
\theta_a,\theta_b \in [\,\theta(1-\rho),\ \theta(1+\rho)\,].
\]
Moreover, the denominator becomes
\[
\Bigl(\frac{v_a^2-v_b^2}{\sigma^2}\Bigr)^2=\Bigl(\frac{\theta_a-\theta_b}{\theta}\Bigr)^2.
\]

\paragraph{Step 1: lower bound.}
Fix any selector $D\in\mathfrak S_\rho(\sigma)$. For any $v_a,v_b\in B_\rho(\sigma)$,
$D(v_a)\in\mathcal H$ has variance $\theta_a$ and $D(v_b)\in\mathcal H$ has variance $\theta_b$. Therefore
\[
\KL(D(v_a)\,\|\,D(v_b))
\ \ge\
\inf_{\substack{Q\in\mathcal H:\\ \Var(Q)=\theta_a}}
\inf_{\substack{P\in\mathcal H:\\ \Var(P)=\theta_b}}
\KL(Q\,\|\,P).
\]
Take $\Delta:=\theta_a-\theta_b$. By Assumption~\ref{as:local-KL}, for every $\varepsilon>0$ there exists
$\delta_\varepsilon>0$ such that if $|\Delta|\le\delta_\varepsilon$ then
\[
\inf_{\Var(P)=\theta_b}\inf_{\Var(Q)=\theta_a}\KL(Q\|P)
\ \ge\
\Bigl(\frac{I_{\text{Var}}(\theta_b)}{2}-\varepsilon\Bigr)\Delta^2.
\]
Since $\theta_b\in[\theta(1-\rho),\theta(1+\rho)]$ and $I_{\text{Var}}(\cdot)$ is finite at $\theta$ under the assumption,
we may choose $\rho$ small enough so that $I_{\text{Var}}(\theta_b)=I_{\text{Var}}(\theta)\,(1+o_\rho(1))$ uniformly over
$\theta_b$ in this interval. Hence, uniformly for all
$v_a,v_b\in B_\rho(\sigma)$ with $|\theta_a-\theta_b|\le\delta_\varepsilon$,
\[
\KL(D(v_a)\,\|\,D(v_b))
\ \ge\
\Bigl(\frac{I_{\text{Var}}(\theta)}{2}-\varepsilon\Bigr)(\theta_a-\theta_b)^2\cdot(1+o_\rho(1)).
\]
Dividing by $\bigl((\theta_a-\theta_b)/\theta\bigr)^2$ gives
\[
\frac{\KL(D(v_a)\,\|\,D(v_b))}{\bigl((\theta_a-\theta_b)/\theta\bigr)^2}
\ \ge\
\theta^2\Bigl(\frac{I_{\text{Var}}(\theta)}{2}-\varepsilon\Bigr)\cdot(1+o_\rho(1)).
\]
Taking the supremum over $v_a,v_b\in B_\rho(\sigma)$ (restricting to pairs with sufficiently small $|\theta_a-\theta_b|$ is
allowed since the supremum is over a larger set), then taking the infimum over selectors $D$, yields
\[
\mathrm{VLC}_\rho(\sigma\mid\mathcal H)
\ \ge\
\theta^2\Bigl(\frac{I_{\text{Var}}(\theta)}{2}-\varepsilon\Bigr)\cdot(1+o_\rho(1)).
\]
Finally, send first $\rho\downarrow 0$ (so $o_\rho(1)\to 0$) and then $\varepsilon\downarrow 0$ to obtain
\begin{equation}\label{eq:vlc-lb-proof}
\liminf_{\rho\downarrow 0}\ \mathrm{VLC}_\rho(\sigma\mid\mathcal H)
\ \ge\
\frac{\theta^2}{2}\,I_{\text{Var}}(\theta)
\ =\
\frac{\sigma^4}{2}\,I_{\text{Var}}(\sigma^2).
\end{equation}

\paragraph{Step 2: upper bound.}
We show the matching upper bound by constructing a selector whose KL matches the quadratic upper expansion.
Fix $\varepsilon>0$. By Assumption~\ref{as:local-KL}, there exists $\delta_\varepsilon>0$ such that for all
$|\Delta|\le\delta_\varepsilon$,
\[
\inf_{\substack{P\in\mathcal H:\\ \Var(P)=\theta}}
\inf_{\substack{Q\in\mathcal H:\\ \Var(Q)=\theta+\Delta}}
\KL(Q\|P)
\ \le\
\Bigl(\frac{I_{\text{Var}}(\theta)}{2}+\varepsilon\Bigr)\Delta^2.
\]
Moreover, the assumption provides a $C^2$ curve $\{P_\vartheta:\vartheta\in(\theta-\eta,\theta+\eta)\}\subset\mathcal H$
with $\Var(P_\vartheta)=\vartheta$ that attains this upper bound up to $o(\Delta^2)$. Choose $\rho_0>0$ such that
$\theta(1+\rho_0)<\theta+\eta$ and $\theta(1-\rho_0)>\theta-\eta$. For any $\rho\in(0,\rho_0)$, define a selector
$D\in\mathfrak S_\rho(\sigma)$ by
\[
D(v)\;:=\;P_{v^2},\qquad v\in B_\rho(\sigma).
\]
Then for any $v_a,v_b\in B_\rho(\sigma)$, writing $\Delta=\theta_a-\theta_b$, the $C^2$ curve property yields
\[
\KL(D(v_a)\|D(v_b))
=
\KL(P_{\theta_a}\|P_{\theta_b})
\ \le\
\Bigl(\frac{I_{\text{Var}}(\theta)}{2}+\varepsilon\Bigr)\Delta^2\cdot(1+o_\rho(1)),
\]
uniformly over $v_a,v_b\in B_\rho(\sigma)$ as $\rho\downarrow 0$. Dividing by $\bigl(\Delta/\theta\bigr)^2$ gives
\[
\sup_{v_a,v_b\in B_\rho(\sigma)}
\frac{\KL(D(v_a)\|D(v_b))}
{\bigl((v_a^2-v_b^2)/\sigma^2\bigr)^2}
\ \le\
\theta^2\Bigl(\frac{I_{\text{Var}}(\theta)}{2}+\varepsilon\Bigr)\cdot(1+o_\rho(1)).
\]
Taking the infimum over selectors yields
\[
\mathrm{VLC}_\rho(\sigma\mid\mathcal H)
\ \le\
\theta^2\Bigl(\frac{I_{\text{Var}}(\theta)}{2}+\varepsilon\Bigr)\cdot(1+o_\rho(1)).
\]
Sending $\rho\downarrow 0$ and then $\varepsilon\downarrow 0$ gives
\begin{equation}\label{eq:vlc-ub-proof}
\limsup_{\rho\downarrow 0}\ \mathrm{VLC}_\rho(\sigma\mid\mathcal H)
\ \le\
\frac{\theta^2}{2}\,I_{\text{Var}}(\theta)
\ =\
\frac{\sigma^4}{2}\,I_{\text{Var}}(\sigma^2).
\end{equation}

\paragraph{Step 3: conclude.}
Combining \eqref{eq:vlc-lb-proof} and \eqref{eq:vlc-ub-proof} yields
\[
\mathrm{VLC}_{\rho}(\sigma\mid\mathcal H)
\;=\;
\frac{\sigma^{4}}{2}\, I_{\text{Var}}(\sigma^2)\,\bigl(1+o_\rho(1)\bigr),
\]
for $\rho$ small enough, where $o_\rho(1)\to 0$ as $\rho\downarrow 0$. This proves the lemma.
\end{proof}

\begin{proof}{\textit{Proof of Lemma~\ref{lem:identification}.}}
Fix $i\ge 1$ and $\vct\lambda\in\Delta_\rho(\vct\lambda^\star)$, and set
\[
\vct v \;:=\; \vct\sigma(\vct\lambda)\;=\;\|\vct\sigma\|_2\,\sqrt{\vct\lambda}.
\]
Since $\vct\lambda^\star=\mathcal N(\vct\sigma)$, we have $\sigma_k^2=\|\vct\sigma\|_2^2\,\lambda_k^\star$ for each
$k$, hence
\[
\Bigl|\frac{v_k^2-\sigma_k^2}{\sigma_k^2}\Bigr|
\;=\;
\Bigl|\frac{\lambda_k}{\lambda_k^\star}-1\Bigr|
\;\le\;\rho,
\qquad k\in[d],
\]
so $v_k\in B_\rho(\sigma_k)$ and the selector $D_k^{(i)}(v_k)$ is well-defined. By construction,
\[
D_i(\vct\lambda)
\;=\;
D^{(i)}(\vct v)
\;=\;
\bigl(D_1^{(i)}(v_1),\dots,D_d^{(i)}(v_d)\bigr),
\]
and each selector satisfies $\Var(D_k^{(i)}(v_k))=v_k^2$. Therefore
$\vct\sigma\!\bigl(D_i(\vct\lambda)\bigr)=\vct v$, and hence
\[
\mathcal N\!\bigl(\vct\sigma(D_i(\vct\lambda))\bigr)
\;=\;
\mathcal N(\vct v)
\;=\;
\frac{\vct v^{\odot 2}}{\|\vct v\|_2^2}
\;=\;
\frac{\|\vct\sigma\|_2^2\,\vct\lambda}{\|\vct\sigma\|_2^2}
\;=\;
\vct\lambda,
\]
where we used $\|\vct v\|_2^2=\|\vct\sigma\|_2^2\sum_{k=1}^d\lambda_k=\|\vct\sigma\|_2^2$.

Next, since $\vct\lambda\in\Delta_\rho(\vct\lambda^\star)$ and
$\mathcal N(\vct\sigma(D_i(\vct\lambda)))=\vct\lambda$, we have
$D_i(\vct\lambda)\in\mathcal H_\rho(\vct\sigma)$ (and in fact also
$R^\star(\vct\sigma(D_i(\vct\lambda)))=\|\vct\sigma\|_2^2/T=R^\star(\vct\sigma)$ by the identity above). Hence,
for any policy $\pi$ and any $i\ge 1$,
\[
\sup_{\vct D\in\mathcal H_{\rho}(\vct\sigma)}\Reg(\pi,\vct D)
\;\ge\;
\sup_{\vct\lambda\in\Delta_\rho(\vct\lambda^\star)}\Reg\!\bigl(\pi, D_i(\vct\lambda)\bigr).
\]
Taking $\limsup_{i\to\infty}$ yields the second display in the lemma. Applying $\inf_\pi$ to both sides preserves
the inequality, and restricting the infimum to $\Pi^{\mathrm{sym}}_{\rho,\tau}(\vct\sigma)$ gives the final
displayed bound for $\mathcal V_{\rho,\tau}(\vct\sigma)$.
\end{proof}

\begin{proof}{\textit{Proof of Lemma~\ref{lem:feasibility_app}.}}
Fix $\vct\lambda^\star\in\Delta_d$ and $\rho\in(0,1)$, and let $x\in\mathcal X$. Write
\[
\vct u_x:=\vct A x,
\qquad\text{so that}\qquad
\vct\lambda_x=\vct\lambda^\star\odot\Bigl(\one+\frac{\rho}{\sqrt d}\,\vct u_x\Bigr).
\]
Since $\vct\lambda^\star\in\Delta_d$ has strictly positive coordinates (here $\vct\lambda^\star=\mathcal N(\vct\sigma)$ with
$\vct\sigma\in\R_+^d$), the componentwise ratio is well-defined and satisfies
\begin{equation}\label{eq:ratio_feasibility}
\frac{\vct\lambda_x}{\vct\lambda^\star}
=
\one+\frac{\rho}{\sqrt d}\,\vct u_x.
\end{equation}

\smallskip
\noindent\textbf{(i)$\Rightarrow$(ii).}
Assume $\vct\lambda_x\in\Delta_\rho(\vct\lambda^\star)$. By definition,
\[
\Bigl\|\frac{\vct\lambda_x}{\vct\lambda^\star}-\one\Bigr\|_2\le\rho.
\]
Using \eqref{eq:ratio_feasibility}, this is equivalent to
\[
\frac{\rho}{\sqrt d}\,\|\vct u_x\|_2 \le \rho,
\qquad\text{i.e.}\qquad
\|\vct u_x\|_2\le \sqrt d.
\]
Moreover, $\vct\lambda_x\in\Delta_d$ implies $\langle \one,\vct\lambda_x\rangle=1$. Expanding the sum constraint,
\[
\langle \one,\vct\lambda_x\rangle
=
\Big\langle \one,\vct\lambda^\star\odot\Bigl(\one+\frac{\rho}{\sqrt d}\,\vct u_x\Bigr)\Big\rangle
=
\langle \one,\vct\lambda^\star\rangle+\frac{\rho}{\sqrt d}\,\langle \vct\lambda^\star,\vct u_x\rangle
=
1+\frac{\rho}{\sqrt d}\,\langle \vct\lambda^\star,\vct u_x\rangle,
\]
so $\langle \one,\vct\lambda_x\rangle=1$ is equivalent to $\langle \vct\lambda^\star,\vct u_x\rangle=0$.

\smallskip
\noindent\textbf{(ii)$\Rightarrow$(i).}
Assume $\langle \vct\lambda^\star,\vct u_x\rangle=0$ and $\|\vct u_x\|_2\le \sqrt d$.
Then \eqref{eq:ratio_feasibility} gives
\[
\Bigl\|\frac{\vct\lambda_x}{\vct\lambda^\star}-\one\Bigr\|_2
=
\frac{\rho}{\sqrt d}\,\|\vct u_x\|_2
\le \rho,
\]
so the $\ell_\infty$ neighborhood constraint holds. Next, the same expansion as above yields
\[
\langle \one,\vct\lambda_x\rangle
=
1+\frac{\rho}{\sqrt d}\,\langle \vct\lambda^\star,\vct u_x\rangle
=
1,
\]
so $\vct\lambda_x$ has unit sum. Finally, for each coordinate $k$,
\[
1+\frac{\rho}{\sqrt d}u_{x,k}\ \ge\ 1-\rho\ \ge\ 0
\qquad(\text{since }\rho\in(0,1)\text{ and }u_{x,k}\ge-\|\vct u_x\|_2\ge-\sqrt d),
\]
and since $\lambda_k^\star\ge 0$, this implies $(\lambda_x)_k\ge 0$ for all $k$. Therefore $\vct\lambda_x\in\Delta_d$.
Combining $\vct\lambda_x\in\Delta_d$ with the neighborhood bound proves $\vct\lambda_x\in\Delta_\rho(\vct\lambda^\star)$.

\smallskip
The final statement follows immediately: if $\vct A\mathcal X\subset(\vct\lambda^\star)^\perp$, i.e.
$\langle \vct\lambda^\star,\vct A x\rangle=0$ for all $x\in\mathcal X$, and
$\sup_{x\in\mathcal X}\|\vct A x\|_2\le \sqrt d$, then condition (ii) holds for all $x\in\mathcal X$, hence
$\{\vct\lambda_x\}_{x\in\mathcal X}\subset\Delta_\rho(\vct\lambda^\star)$.
\end{proof}

\begin{proof}{\textit{Proof of Lemma~\ref{lem:separation_app}.}}
Fix a policy $\pi$, an index $i\ge 1$, and a label $x\in\mathcal X$. Let $\vct\sigma_x:=\vct\sigma(\vct\lambda_x)$
denote the standard-deviation vector associated with $\vct\lambda_x$, so that by construction of the instance
generator,
\[
\vct\sigma\bigl(\vct D_i(x)\bigr)=\vct\sigma_x,
\qquad\text{and}\qquad
\mathcal N(\vct\sigma_x)=\vct\lambda_x.
\]
Run $\pi$ on $\vct D_i(x)$ and let $\vct n=(n_1,\dots,n_d)$ be the terminal pull counts, with terminal decision
$\bar{\vct\lambda}:=\vct n/T$. For this realized $\vct n$, the risk is
\[
R(\vct n;\vct\sigma_x)=\max_{k\in[d]}\frac{\sigma_{x,k}^2}{n_k},
\]
with the convention $R(\vct n;\vct\sigma_x)=+\infty$ if $n_k=0$ for some $k$ with $\sigma_{x,k}>0$. Since
$\vct\lambda_x=\mathcal N(\vct\sigma_x)$, we can write $\sigma_{x,k}^2=\|\vct\sigma_x\|_2^2\,\lambda_{x,k}$ for
each $k$, hence
\[
R(\vct n;\vct\sigma_x)
=
\|\vct\sigma_x\|_2^2\cdot \max_{k\in[d]}\frac{\lambda_{x,k}}{n_k}.
\]
Moreover,
\[
R^\star(\vct\sigma_x)=\frac{\|\vct\sigma_x\|_2^2}{T},
\]
so for every realized $\vct n$,
\[
\frac{R(\vct n;\vct\sigma_x)}{R^\star(\vct\sigma_x)}-1
=
\max_{k\in[d]}\frac{T\,\lambda_{x,k}}{n_k}-1
=
\max_{k\in[d]}\frac{\lambda_{x,k}}{\bar\lambda_k}-1
=
r(\bar{\vct\lambda}\,\|\,\vct\lambda_x),
\]
where in the last step we used the definition of $r$ in \eqref{eq:def-r-app} (with the convention
$r(\bar{\vct\lambda}\,\|\,\vct\lambda_x)=+\infty$ if $\bar\lambda_k=0$ for some $k$).

Taking expectation over the randomness of the policy and the samples (equivalently, over the induced random
terminal decision $\bar{\vct\lambda}$) yields
\[
\Reg(\pi,\vct D_i(x))
=
\E\!\left[\frac{R(\vct n;\vct\sigma_x)}{R^\star(\vct\sigma_x)}-1\right]
=
\E\!\left[r(\bar{\vct\lambda}\,\|\,\vct\lambda_x)\right],
\]
which is the claimed separation identity.
\end{proof}

\begin{proof}{\textit{Proof of Lemma~\ref{lem:projection-reduces-regret-app}.}}
Fix the true label $x\in\mathcal X$ and a realization of the terminal decision $\bar{\vct\lambda}\in\Delta_d$.
Let $\bar x\in \text{argmin}_{x'\in\mathcal X} r(\bar{\vct\lambda}\|\vct\lambda_{x'})$ be the minimum-mismatch decoder
(Definition~\ref{def:decoder-app}). By definition of $\tilde r$ (Definition~\ref{def:symm_dissimilarity_app}),
for any $\vct\lambda\in\Delta_d$,
\[
\max\!\Big\{ r(\vct\lambda\|\vct\lambda_x),\ r(\vct\lambda\|\vct\lambda_{\bar x})\Big\}
\ \ge\
\tilde r(\vct\lambda_x,\vct\lambda_{\bar x}).
\]
Applying this inequality with $\vct\lambda=\bar{\vct\lambda}$ yields
\[
\max\!\Big\{ r(\bar{\vct\lambda}\|\vct\lambda_x),\ r(\bar{\vct\lambda}\|\vct\lambda_{\bar x})\Big\}
\ \ge\
\tilde r(\vct\lambda_x,\vct\lambda_{\bar x}).
\]
Moreover, since $\bar x$ minimizes $r(\bar{\vct\lambda}\|\vct\lambda_{x'})$ over $x'\in\mathcal X$, we have
$r(\bar{\vct\lambda}\|\vct\lambda_{\bar x})\le r(\bar{\vct\lambda}\|\vct\lambda_x)$, hence the maximum on the
left-hand side equals $r(\bar{\vct\lambda}\|\vct\lambda_x)$. Therefore,
\[
r(\bar{\vct\lambda}\|\vct\lambda_x)\ \ge\ \tilde r(\vct\lambda_x,\vct\lambda_{\bar x}).
\]
Taking expectation over the randomness of $\bar{\vct\lambda}$ (induced by the policy and samples) gives the first
inequality in the lemma:
\[
\E\!\left[r(\bar{\vct\lambda}\|\vct\lambda_x)\right]\ \ge\ \E\!\left[\tilde r(\vct\lambda_x,\vct\lambda_{\bar x})\right].
\]

For the second inequality, decompose according to the event $\{\bar x=x\}$:
\[
\E\!\left[\tilde r(\vct\lambda_x,\vct\lambda_{\bar x})\right]
=
\E\!\left[\tilde r(\vct\lambda_x,\vct\lambda_{\bar x})\,\mathbf 1\{\bar x=x\}\right]
+
\E\!\left[\tilde r(\vct\lambda_x,\vct\lambda_{\bar x})\,\mathbf 1\{\bar x\neq x\}\right].
\]
Since $\tilde r(\vct\lambda_x,\vct\lambda_x)=0$ (take $\vct\lambda=\vct\lambda_x$ in the defining infimum), the
first term is $0$. On the event $\{\bar x\neq x\}$ we have
$\tilde r(\vct\lambda_x,\vct\lambda_{\bar x})\ge \min_{b\neq x}\tilde r(\vct\lambda_x,\vct\lambda_b)$, hence
\[
\E\!\left[\tilde r(\vct\lambda_x,\vct\lambda_{\bar x})\right]
\ \ge\
\Pr(\bar x\neq x)\cdot \min_{b\neq x}\tilde r(\vct\lambda_x,\vct\lambda_b),
\]
which is the desired bound.
\end{proof}

\begin{proof}{\textit{Proof of Lemma~\ref{lem:rtilde_l1_app}.}}
Fix $\vct\lambda_a,\vct\lambda_b\in\Delta_d$. Recall that
\[
r(\vct\lambda\|\vct\lambda_a)=\left\|\frac{\vct\lambda_a}{\vct\lambda}\right\|_\infty-1
=\max_{k\in[d]}\frac{\lambda_{k,a}}{\lambda_k}-1,
\]
with the convention $r(\vct\lambda\|\vct\lambda_a)=+\infty$ if $\lambda_k=0$ for some $k$ with $\lambda_{k,a}>0$.
We start with the following elementary fact:
\begin{equation}\label{eq:fact-simplex-lowerbounds}
\forall \vct u\in\R^d_{>0},\qquad \Delta_d\cap(\vct u,+\infty)\neq\emptyset
\quad\Longleftrightarrow\quad
\sum_{k=1}^d u_k<1,
\end{equation}
where $(\vct u,+\infty):=\{\vct\lambda\in\R^d:\ \lambda_k>u_k\ \forall k\}$.

Let $t\ge 0$. Then
\begin{align*}
\tilde r(\vct\lambda_a,\vct\lambda_b) < t
&\iff \inf_{\vct\lambda\in\Delta_d}\max\!\bigl(r(\vct\lambda\|\vct\lambda_a),\,r(\vct\lambda\|\vct\lambda_b)\bigr) < t\\
&\iff \exists\,\vct\lambda\in\Delta_d:\ \max\!\bigl(r(\vct\lambda\|\vct\lambda_a),\,r(\vct\lambda\|\vct\lambda_b)\bigr) < t\\
&\iff \exists\,\vct\lambda\in\Delta_d:\ \max_{k\in[d],\,u\in\{a,b\}}\frac{\lambda_{k,u}}{\lambda_k} < 1+t\\
&\iff \exists\,\vct\lambda\in\Delta_d:\ \forall k\in[d],\ \frac{1}{1+t}\max(\lambda_{k,a},\lambda_{k,b}) < \lambda_k\\
&\iff \frac{1}{1+t}\sum_{k=1}^d \max(\lambda_{k,a},\lambda_{k,b}) < 1 \qquad\text{(by \eqref{eq:fact-simplex-lowerbounds})}\\
&\iff \Big(\sum_{k=1}^d \max(\lambda_{k,a},\lambda_{k,b})\Big) - 1 < t.
\end{align*}
Since this equivalence holds for all $t\ge 0$, it follows that
\[
\tilde r(\vct\lambda_a,\vct\lambda_b)
=
\Big(\sum_{k=1}^d \max(\lambda_{k,a},\lambda_{k,b})\Big) - 1.
\]
Finally, using $\max(x,y)=\frac{x+y+|x-y|}{2}$ and $\sum_k\lambda_{k,a}=\sum_k\lambda_{k,b}=1$, we obtain
\[
\sum_{k=1}^d \max(\lambda_{k,a},\lambda_{k,b})
=
\frac12\sum_{k=1}^d\bigl(\lambda_{k,a}+\lambda_{k,b}+|\lambda_{k,a}-\lambda_{k,b}|\bigr)
=
1+\frac12\|\vct\lambda_a-\vct\lambda_b\|_1.
\]
Therefore,
\[
\tilde r(\vct\lambda_a,\vct\lambda_b)=\frac12\|\vct\lambda_a-\vct\lambda_b\|_1,
\]
as claimed.
\end{proof}

\begin{proof}{\textit{Proof of Lemma~\ref{lem:fano_bound_app}.}}
Fix a nonempty $\mathcal X\subset\{\pm1\}^d\cap\one^\perp$, a matrix $\vct A\in\R^{d\times d}$, and a policy
$\pi\in\Pi^{\mathrm{sym}}_{\rho,\tau}(\vct\sigma)$. Let $x\sim\mathcal U(\mathcal X)$ and let $\mathcal F_T$ denote
the bandit transcript generated by running $\pi$ on the instance $\vct D_i(x)$. For any decoder $\bar x=\bar x(\mathcal F_T)$,
Fano's inequality gives
\begin{equation}\label{eq:fano_step1}
\Pr(\bar x\neq x)
\ \ge\
1-\frac{\log 2 + I(x;\mathcal F_T)}{\log|\mathcal X|}.
\end{equation}
Moreover, since $x$ is uniform, the mutual information admits the standard upper bound
\begin{equation}\label{eq:mi_pairwise}
I(x;\mathcal F_T)
\ \le\
\frac{1}{|\mathcal X|^2}\sum_{x,y\in\mathcal X}\KL\!\Big((\pi,\vct D_i(x))\,\Big\|\,(\pi,\vct D_i(y))\Big),
\end{equation}
where $(\pi,\vct D_i(x))$ denotes the law of the transcript under policy $\pi$ and instance $\vct D_i(x)$.

\paragraph{Step 1: bandit KL decomposition.}
For fixed $x,y\in\mathcal X$, the canonical KL decomposition for bandit transcripts yields
\begin{equation}\label{eq:bandit_kl_decomp}
\KL\!\Big((\pi,\vct D_i(x))\,\Big\|\,(\pi,\vct D_i(y))\Big)
=
\sum_{k=1}^d \E_{\vct n\sim(\pi,\vct D_i(x))}[n_k]\;
\KL\!\Big(D^{(i)}_{k}(\vct\lambda_x)\,\Big\|\,D^{(i)}_{k}(\vct\lambda_y)\Big),
\end{equation}
where $D^{(i)}_{k}(\vct\lambda_x)$ denotes the $k$-th arm distribution in $\vct D_i(\vct\lambda_x)$.

\paragraph{Step 2: controlling pull counts via tolerance.}
Let $\bar{\vct\lambda}:=\vct n/T$ be the terminal decision under $(\pi,\vct D_i(x))$. By Lemma~\ref{lem:separation_app},
\[
\Reg(\pi,\vct D_i(x))=\E\!\left[r(\bar{\vct\lambda}\,\|\,\vct\lambda_x)\right]\le \tau.
\]
Since $r(\cdot\|\vct\lambda_x)$ is convex in its first argument, Jensen's inequality implies
\[
r\!\big(\E[\bar{\vct\lambda}]\,\big\|\,\vct\lambda_x\big)\ \le\ \E\!\left[r(\bar{\vct\lambda}\,\|\,\vct\lambda_x)\right]\ \le\ \tau.
\]
Unpacking the definition $r(\vct u\|\vct v)=\|\vct v/\vct u\|_\infty-1$, this gives, for every $k\in[d]$,
\[
\frac{(\lambda_x)_k}{\E[\bar\lambda_k]}-1\ \le\ \tau
\qquad\Longrightarrow\qquad
\E[\bar\lambda_k]\ \ge\ \frac{(\lambda_x)_k}{1+\tau}.
\]
Using $\sum_{k}\E[\bar\lambda_k]=1$, we obtain the complementary upper bound
\[
\E[\bar\lambda_k]
=
1-\sum_{j\neq k}\E[\bar\lambda_j]
\ \le\
1-\frac{1}{1+\tau}\sum_{j\neq k}(\lambda_x)_j
=
\frac{(\lambda_x)_k+\tau}{1+\tau}
\ \le\ (\lambda_x)_k+\tau.
\]
Since $\vct\lambda_x\in\Delta_\rho(\vct\lambda^\star)$, we have $(\lambda_x)_k\le (1+\rho)\lambda_k^\star$, hence
\begin{equation}\label{eq:lambda_bar_control}
\E[\bar\lambda_k]\ \le\ (1+\rho)\lambda_k^\star+\tau
\ \le\
(1+\varepsilon(\rho,\tau))\,\lambda_k^\star,
\end{equation}
where one may take, e.g., $\varepsilon(\rho,\tau):=\rho+\tau/\min_{j\in[d]}\lambda_j^\star$, which satisfies
$\varepsilon(\rho,\tau)\to 0$ as $(\rho,\tau)\to(0,0)$ for fixed $\vct\lambda^\star\in\Delta_d$.

Combining \eqref{eq:lambda_bar_control} with $\E[n_k]=T\,\E[\bar\lambda_k]$ yields
\begin{equation}\label{eq:nk_control}
\E_{\vct n\sim(\pi,\vct D_i(x))}[n_k]\ \le\ T\,\lambda_k^\star\,(1+\varepsilon(\rho,\tau)).
\end{equation}

\paragraph{Step 3: controlling arm-level KL via $\mathrm{VLC}_\rho$.}
By construction of the instance generator, the variance of arm $k$ under $\vct D_i(\vct\lambda)$ is
$\sigma_k(\vct\lambda)^2=\|\vct\sigma\|_2^2\,\lambda_k$. Using the selectors in \eqref{eq:optimalrepresentation-app}
and the definition of $\mathrm{VLC}_\rho(\sigma_k\mid\mathcal H)$, we have
\[
\limsup_{i\to\infty}
\KL\!\Big(D^{(i)}_{k}(\vct\lambda_x)\,\Big\|\,D^{(i)}_{k}(\vct\lambda_y)\Big)
\ \le\
\mathrm{VLC}_\rho(\sigma_k\mid\mathcal H)\,
\Bigl(\frac{\sigma_k(\vct\lambda_x)^2-\sigma_k(\vct\lambda_y)^2}{\sigma_k^2}\Bigr)^2.
\]
Since $\sigma_k^2=\|\vct\sigma\|_2^2\,\lambda_k^\star$ and $\sigma_k(\vct\lambda)^2=\|\vct\sigma\|_2^2\,\lambda_k$, we get
\[
\frac{\sigma_k(\vct\lambda_x)^2-\sigma_k(\vct\lambda_y)^2}{\sigma_k^2}
=
\frac{(\lambda_x)_k-(\lambda_y)_k}{\lambda_k^\star}.
\]
Under the parametrization $\vct\lambda_x=\vct\lambda^\star\odot(\one+\frac{\rho}{\sqrt d}\vct A x)$,
\[
(\lambda_x)_k-(\lambda_y)_k
=
\lambda_k^\star\cdot\frac{\rho}{\sqrt d}\,[\vct A(x-y)]_k,
\]
so the ratio equals $\frac{\rho}{\sqrt d}[\vct A(x-y)]_k$, and therefore
\begin{equation}\label{eq:kl_arm_bound}
\limsup_{i\to\infty}
\KL\!\Big(D^{(i)}_{k}(\vct\lambda_x)\,\Big\|\,D^{(i)}_{k}(\vct\lambda_y)\Big)
\ \le\
\frac{\rho^2}{d}\,\mathrm{VLC}_\rho(\sigma_k\mid\mathcal H)\,[\vct A(x-y)]_k^2.
\end{equation}

\paragraph{Step 4: averaging and Frobenius form.}
Plugging \eqref{eq:nk_control} and \eqref{eq:kl_arm_bound} into \eqref{eq:bandit_kl_decomp} gives
\[
\limsup_{i\to\infty}
\KL\!\Big((\pi,\vct D_i(x))\,\Big\|\,(\pi,\vct D_i(y))\Big)
\ \le\
\frac{T\rho^2}{d}\,(1+\varepsilon(\rho,\tau))
\sum_{k=1}^d \lambda_k^\star\,\mathrm{VLC}_\rho(\sigma_k\mid\mathcal H)\,[\vct A(x-y)]_k^2.
\]
With $\vct\Gamma:=\mathrm{diag}(\lambda_k^\star\,\mathrm{VLC}_\rho(\sigma_k\mid\mathcal H))_{k=1}^d$, the sum equals
$(x-y)^\top \vct A^\top \vct\Gamma \vct A (x-y)$. Averaging over $x,y\sim\mathcal U(\mathcal X)$ i.i.d., and using
$\E[(x-y)(x-y)^\top]=2\,\Cov(\mathcal U(\mathcal X))=2\vct\Sigma$, we obtain
\begin{align*}
\limsup_{i\to\infty}\ \frac{1}{|\mathcal X|^2}\sum_{x,y\in\mathcal X}
\KL\!\Big((\pi,\vct D_i(x))\,\Big\|\,(\pi,\vct D_i(y))\Big)
&\le
\frac{T\rho^2}{d}\,(1+\varepsilon(\rho,\tau))\cdot 2\,\Tr(\vct\Gamma\,\vct A\,\vct\Sigma\,\vct A^\top)\\
&=
\frac{2T\rho^2}{d}\,(1+\varepsilon(\rho,\tau))\,
\big\|\sqrt{\vct\Gamma}\,\vct A\,\sqrt{\vct\Sigma}\big\|_F^2.
\end{align*}
Combining this bound with \eqref{eq:fano_step1}--\eqref{eq:mi_pairwise} yields the stated inequality.
\end{proof}

\begin{proof}{\textit{Proof of Lemma~\ref{lem:decision_cost_app}.}}
Fix $i\ge 1$ and $x\in\mathcal X$. On the constructed family we have
\[
\vct\lambda_x
=
\vct\lambda^\star\odot\Bigl(\one+\frac{\rho}{\sqrt d}\,\vct A x\Bigr),
\qquad
\vct\lambda_{\bar x}
=
\vct\lambda^\star\odot\Bigl(\one+\frac{\rho}{\sqrt d}\,\vct A \bar x\Bigr).
\]
Therefore
\[
\vct\lambda_x-\vct\lambda_{\bar x}
=
\frac{\rho}{\sqrt d}\,\vct\lambda^\star\odot\bigl(\vct A(x-\bar x)\bigr)
=
\frac{\rho}{\sqrt d}\,\vct\Lambda\,\vct A\,(x-\bar x),
\]
where $\vct\Lambda=\text{diag}(\vct\lambda^\star)$. By Lemma~\ref{lem:rtilde_l1_app},
\[
\tilde r(\vct\lambda_x,\vct\lambda_{\bar x})
=
\frac12\|\vct\lambda_x-\vct\lambda_{\bar x}\|_1
=
\frac{\rho}{2\sqrt d}\,\big\|\vct\Lambda\,\vct A\,(x-\bar x)\big\|_1.
\]
Next, recall $\vct P_{\neq}:=\text{diag}(\mathbf 1\{x_k\neq \bar x_k\})$. For each coordinate $k$,
if $x_k=\bar x_k$ then $(x-\bar x)_k=0$, and if $x_k\neq \bar x_k$ then necessarily $\bar x_k=-x_k$ and
$(x-\bar x)_k=2x_k$. Hence
\[
x-\bar x \;=\; 2\,\vct P_{\neq}x,
\]
and plugging this into the previous display yields, on the event $\{\bar x\neq x\}$ (and in fact always),
\[
\tilde r(\vct\lambda_x,\vct\lambda_{\bar x})
=
\frac{\rho}{2\sqrt d}\,\big\|\vct\Lambda\,\vct A\,(2\,\vct P_{\neq}x)\big\|_1
=
\frac{\rho}{\sqrt d}\,\big\|\vct\Lambda\,\vct A\,(\vct P_{\neq}x)\big\|_1.
\]
Taking conditional expectation given $\{\bar x\neq x\}$ gives the second display.
\end{proof}

\begin{proof}{\textit{Proof of Lemma~\ref{lem:gv_design_tuning}.}}
Let $\mathcal E$ be the set of subsets of $\{-1,1\}^d$ satisfying both orthogonality and separation, i.e.,
\[
\mathcal E
:=\Bigl\{\mathcal X\subset\{-1,1\}^d:\ \mathcal X\subset \one^\perp,\ \min_{x\neq x'\in\mathcal X} d_H(x,x')\ge \frac d4\Bigr\}.
\]
Let $m$ be the largest size in $\mathcal E$, i.e.,
\[
m:=\max_{\mathcal X\in\mathcal E}|\mathcal X|.
\]
Proving the lemma boils down to proving $\log m = \Omega(d)$. First, notice that $\mathcal E$ is non-empty, since the singleton
\[
\Bigl\{(1,\ldots,1,-1,\ldots,-1)\Bigr\}
\]
(consisting of $d/2$ ones and $d/2$ minus ones) lies in $\mathcal E$. Let $\mathcal S\in\mathcal E$ be a set of
maximal size. We introduce the balls associated with the $0$-norm $\|x\|_0=\sum_{k=1}^d\mathbf 1\{x_k\neq 0\}$:
for $x\in\{-1,1\}^d$ and $r>0$,
\[
B_0(x,r):=\Bigl\{y\in\{-1,1\}^d:\ \|x-y\|_0<r\Bigr\}.
\]
We claim that
\begin{equation}\label{eq:gv_cover}
\{-1,1\}^d\cap \one^\perp
=
\bigcup_{y\in\mathcal S}\Bigl(B_0(y,d/4)\cap \one^\perp\Bigr).
\end{equation}
The inclusion $\subseteq$ is immediate. Assume for the sake of contradiction that it is strict, i.e., there exists
$x\in\{-1,1\}^d$ with $x\perp\one$ such that
\[
\forall y\in\mathcal S,\qquad \sum_{k=1}^d \mathbf 1\{x_k\neq y_k\}\ge \frac d4,
\]
or equivalently $\min_{y\in\mathcal S} d_H(x,y)\ge d/4$. In particular $x\notin\mathcal S$, and the new set
$\mathcal S':=\mathcal S\cup\{x\}$ still satisfies orthogonality and separation, hence $\mathcal S'\in\mathcal E$,
contradicting the maximality of $\mathcal S$. Therefore \eqref{eq:gv_cover} holds.

Next, using \eqref{eq:gv_cover} and a union bound,
\[
|\{-1,1\}^d\cap \one^\perp|
=
\Bigl|\bigcup_{y\in\mathcal S}\bigl(B_0(y,d/4)\cap \one^\perp\bigr)\Bigr|
\le
\sum_{y\in\mathcal S}\bigl|B_0(y,d/4)\cap \one^\perp\bigr|.
\]
By symmetry of $\{-1,1\}^d$, the sets $B_0(y,d/4)\cap\one^\perp$ have constant size over
$\{-1,1\}^d\cap\one^\perp$. In particular, they have the same size as
$B_0(y_0,d/4)\cap\one^\perp$, where $y_0=(1,\ldots,1,-1,\ldots,-1)$ consists of $d/2$ ones and $d/2$ minus ones.
Therefore,
\[
|\{-1,1\}^d\cap \one^\perp|
\le
|\mathcal S|\cdot |B_0(y_0,d/4)\cap\one^\perp|
=
m\cdot |B_0(y_0,d/4)\cap\one^\perp|,
\]
or equivalently
\[
m \ \ge\ \frac{|\{-1,1\}^d\cap\one^\perp|}{|B_0(y_0,d/4)\cap\one^\perp|}.
\]
It remains to prove that this log-ratio grows linearly in $d$. The numerator can be calculated easily: choosing an
element from $\{-1,1\}^d\cap\one^\perp$ is equivalent to choosing $d/2$ placements for $+1$ and setting the
remaining $d/2$ placements to $-1$, hence
\[
|\{-1,1\}^d\cap\one^\perp|=\binom{d}{d/2}.
\]

We now upper bound the denominator. Let $x\in B_0(y_0,d/4)\cap\one^\perp$ and let
$u:=\|x-y_0\|_0\in\{0,1,\ldots,\lfloor d/4\rfloor\}$. We claim that $u$ must be even. Indeed, for each $k\in[d]$,
\[
\mathbf 1\{x_k\neq y_{0,k}\}=\frac12|x_k-y_{0,k}|\equiv \frac{x_k-y_{0,k}}2 \pmod 2,
\]
so
\[
u=\sum_{k=1}^d \mathbf 1\{x_k\neq y_{0,k}\}
\equiv
\frac12\sum_{k=1}^d(x_k-y_{0,k})
=
\frac12(\langle \one,x\rangle-\langle\one,y_0\rangle)
\equiv 0 \pmod 2,
\]
where we used $x\perp\one$ and $y_0\perp\one$. Hence $2\mid u$, and we write $u=2u'$.

For $u=2u'$, choosing $x$ such that $\|x-y_0\|_0=2u'$ amounts to choosing exactly $u'$ coordinates among the
$+1$ entries of $y_0$ and $u'$ coordinates among the $-1$ entries of $y_0$ to flip, hence there are
$\binom{d/2}{u'}^2$ such vectors. Therefore,
\[
|B_0(y_0,d/4)\cap\one^\perp|
=
\sum_{\substack{u<d/4\\ 2\mid u}}\binom{d/4}{u/2}^2
=
\sum_{u'<d/8}\binom{d/4}{u'}^2
\le
\frac12\sum_{u=0}^{\lfloor d/4\rfloor}\binom{d/4}{u}^2
\le
\binom{d/4}{d/8},
\]
where the last inequality follows from Vandermonde's inequality. Consequently,
\[
m\ \ge\ 2\cdot \frac{\binom{d}{d/2}}{\binom{d/4}{d/8}}
=
\frac{2\,d!\,[(d/8)!]^2}{[(d/2)!]^2\,(d/4)!}.
\]

We show that the lower bound above grows exponentially in $d$. By Stirling's formula,
\[
n! = \sqrt{2\pi n}\Bigl(\frac ne\Bigr)^n\Bigl(1+O\Bigl(\frac1n\Bigr)\Bigr),
\]
so a direct substitution yields
\[
\frac{2\,d!\,[(d/8)!]^2}{[(d/2)!]^2\,(d/4)!}
=
16\Bigl(1+O\Bigl(\frac1d\Bigr)\Bigr)\cdot 2^{3d/4}.
\]
Therefore, there exists $\mathcal X\in\mathcal E$ such that $\log|\mathcal X|=\Omega(d)$, completing the proof.
\end{proof}

\begin{proof}{\textit{Proof of Lemma~\ref{lem:good_representation_tuning}.}}
Fix a policy (and hence a decoder $\bar x\in\mathcal X$) and write
\[
\vct\Sigma:=\Cov(\mathcal U(\mathcal X)),\qquad 
\vct\Lambda:=\text{diag}(\vct\lambda^\star),\qquad
\vct\Gamma:=\text{diag}\!\Big(\lambda_k^\star\,\text{VLC}_\rho(\sigma_k\mid\mathcal H)\Big)_{k=1}^d.
\]
Let $\hat{\vct{\lambda}}:=\vct{\lambda}^\star/\|\vct{\lambda}^\star\|_2$ and define the constrained orthogonal class
\[
\mathsf O(\one\mapsto \hat{\vct\lambda})
\;:=\;
\Bigl\{\vct O\in\mathsf O(d):\ \vct O\,\frac{\one}{\sqrt d}=\hat{\vct\lambda}\Bigr\}.
\]
Let $x\sim\mathcal U(\mathcal X)$ and $\vct O\sim\mathcal U(\mathsf O(\one\mapsto \hat{\vct\lambda}))$ be independent,
and set $\vct A:=\vct S\vct O$, where $\vct S = \text{diag}(\vct s)$ is a diagonal matrix with positive entries. For a realization $(\vct O,x)$ define
\[
U(\vct O,x)
:=
\E\!\left[\big\|\vct\Lambda\vct S\,\vct O(\vct P_{\neq}x)\big\|_1\ \middle|\ \bar x\neq x\right],
\qquad
V(\vct O)
:=
\big\|\sqrt{\vct\Gamma}\vct S\,\vct O\,\sqrt{\vct\Sigma}\big\|_F .
\]

\paragraph{Step 1: averaging the distinguishability term.}
Since $\mathcal X\subset\one^\perp$, we have $\E[x]\in\one^\perp$ and $\vct\Sigma\,\one=\vct{0}$, hence
$\mathrm{range}(\vct\Sigma)\subset\one^\perp$. Using $\| \sqrt{\vct\Gamma} \vct S\,\vct O\,\sqrt{\vct\Sigma}\|_F^2
=\Tr(\vct S^\top\vct\Gamma \vct S\,\vct O\,\vct\Sigma\,\vct O^\top)$ and invariance of the uniform law on
$\mathsf O(\one\mapsto \hat{\vct\lambda})$ on the orthogonal complement $\one^\perp$, we obtain
\[
\E_{\vct O, x}\!\left[V(\vct O)^2\right]
=
\E_{\vct O, x}\!\left[\Tr(\vct S^\top\vct\Gamma\vct S\,\vct O\,\vct\Sigma\,\vct O^\top)\right]
=
\frac{\Tr(\vct S^\top\vct\Gamma \vct S)\Tr(\vct\Sigma)}{d-1}.
\]

\paragraph{Step 2: averaging the decision-cost term.}
On the event $\{\bar x\neq x\}$, since $\bar x\in\mathcal X$ and $\min_{x\neq x'}d_H(x,x')\ge d/4$,
we have $d_H(x,\bar x)\ge d/4$. Moreover,
$x-\bar x=2\vct P_{\neq}x$ and $x-\bar x\in\one^\perp$, so $\vct P_{\neq}x\in\one^\perp$ and
\[
\|\vct P_{\neq}x\|_2=\sqrt{d_H(x,\bar x)}\ \ge\ \sqrt{d/4}.
\]
Fix any $v\in\one^\perp$ and write $r:=\|v\|_2$. Under $\vct O\sim\mathcal U(\mathsf O(\one\mapsto \hat{\vct\lambda}))$,
the random vector $\vct O v$ is uniformly distributed on the sphere of radius $r$ in $\hat{\vct\lambda}^\perp$.
Let $U\sim\mathrm{Unif}(\mathbb S^{d-2})$ and set
\[
\alpha_d:=\E|U_1|=\frac{\Gamma(d/2)}{\sqrt\pi\,\Gamma((d+1)/2)}
=\sqrt{\frac{2}{\pi d}}\,(1+o(1)).
\]
By rotational symmetry in $\hat{\vct\lambda}^\perp$,
\[
\E_{\vct O,x}\bigl[|e_k^\top \vct O v|\bigr]
=
r\,\alpha_d\,\big\| \text{Proj}_{\hat{\vct\lambda}^\perp} e_k\big\|_2
=
r\,\alpha_d\,\sqrt{1-\hat\lambda_k^2}
=
r\,\alpha_d\,\sqrt{1-\left(\lambda_k^\star/\|\vct\lambda^\star\|_2\right)^2}.
\]
Therefore,
\[
\E_{\vct O, x}\bigl[\|\vct\Lambda\vct S\,\vct O v\|_1\bigr]
=
\sum_{k=1}^d \lambda_k^\star s_k\,\E_{\vct O,x}\bigl[|e_k^\top \vct O v|\bigr]
=
r\,\alpha_d\sum_{k=1}^d
\lambda_k^\star s_k\sqrt{1-\left(\lambda_k^\star/\|\vct\lambda^\star\|_2\right)^2}.
\]
Applying this with $v=\vct P_{\neq}x$ and using $r=\|v\|_2\ge\sqrt{d/4}$ on $\{\bar x\neq x\}$ (by definition of $\mathcal{X}$) yields
\[
\E_{\vct O, x}\!\left[U(\vct O,x)\right]
=
\E_{\vct O, x}\E\!\left[\big\|\vct\Lambda\,\vct S\vct O(\vct P_{\neq}x)\big\|_1\ \middle|\ \bar x\neq x\right]
\ \ge\
\frac{\sqrt d}{2}\,\alpha_d
\sum_{k=1}^d
\lambda_k^\star s_k\sqrt{1-\left(\lambda_k^\star/\|\vct\lambda^\star\|_2\right)^2}.
\]
Averaging over $x\sim\mathcal U(\mathcal X)$ preserves this lower bound, hence
\[
\E_{\vct O, x}\!\left[U(\vct O,x)\right]
\ \ge\
\frac{\sqrt d}{2}\,\alpha_d
\sum_{k=1}^d
\lambda_k^\star s_k\sqrt{1-\left(\lambda_k^\star/\|\vct\lambda^\star\|_2\right)^2}
=
\frac{1+o(1)}{\sqrt{2\pi}}
\sum_{k=1}^d
\lambda_k^\star s_k\sqrt{1-\left(\lambda_k^\star/\|\vct\lambda^\star\|_2\right)^2}.
\]

\paragraph{Step 3: existence of a good $\vct A$.}
Let
\[
c
:=
\frac{\E_{\vct O, x}[U(\vct O,x)]}{\sqrt{\E_{\vct O, x}[V(\vct O)^2]}}.
\]
Since $U(\vct O,x)\ge 0$ and $V(\vct O)\ge 0$, we have
$\E[U(\vct O,x)^2]\ge \E[U(\vct O,x)]^2$, and therefore
\[
\E_{\vct O, x}\!\left[U(\vct O,x)^2-c^2V(\vct O)^2\right]
\ge
\E_{\vct O, x}[U(\vct O,x)]^2-c^2\,\E_{\vct O, x}[V(\vct O)^2]
=0.
\]
Hence there exists a realization $\vct O_\star\in \mathsf O(\one\mapsto \hat{\vct\lambda})$
such that $U(\vct O_\star, x)\ge c\,V(\vct O_\star)$, i.e.
\[
\frac{
\E_x\!\left[\big\|\vct\Lambda\vct S\,\vct O_\star(\vct P_{\neq}x)\big\|_1\ \middle|\ \bar x\neq x\right]
}{
\big\|\sqrt{\vct\Gamma}\vct S\,\vct O_\star\,\sqrt{\vct\Sigma}\big\|_F
}
\ \ge\ c \geq \frac{\frac{1+o(1)}{\sqrt{2\pi}}
\sum_{k=1}^d
\lambda_k^\star s_k\sqrt{1-\left(\lambda_k^\star/\|\vct\lambda^\star\|_2\right)^2}}{\sqrt{\frac{\Tr(\vct S^\top\vct\Gamma \vct S)\Tr(\vct\Sigma)}{d-1}}}.
\]
Using $\text{Tr}(\vct S^\top\vct\Gamma\vct S)=\sum_{k=1}^d\lambda_k^\star\text{VLC}_\rho(\sigma_k\mid\mathcal H) s_k^2$ and $\text{Tr}(\Sigma) \leq d$, we obtain

\begin{equation*}
    \frac{
\E\!\left[\big\|\vct\Lambda\vct S\,\vct O_\star(\vct P_{\neq}x)\big\|_1\ \middle|\ \bar x\neq x\right]
}{
\big\|\sqrt{\vct\Gamma}\vct S\,\vct O_\star\,\sqrt{\vct\Sigma}\big\|_F
}
\ \geq \frac{\frac{d-1}{d}\frac{1+o(1)}{\sqrt{2\pi}}
\sum_{k=1}^d
\lambda_k^\star s_k\sqrt{1-\left(\lambda_k^\star/\|\vct\lambda^\star\|_2\right)^2}}{\sqrt{\sum_{k=1}^d\lambda_k^\star\text{VLC}_\rho(\sigma_k\mid\mathcal H) s_k^2}}.
\end{equation*}
Taking the argmax of the RHS above over $\vct{s} > \vct{0}_d$ with $\|\vct{s}\|_2 = 1$ and setting $\vct S_\star = \text{diag}(s_1, \ldots, s_d)$, we obtain
\begin{equation*}
    \frac{
\E\!\left[\big\|\vct\Lambda\vct S_\star\,\vct O_\star(\vct P_{\neq}x)\big\|_1\ \middle|\ \bar x\neq x\right]
}{
\big\|\sqrt{\vct\Gamma}\vct S_\star\,\vct O_\star\,\sqrt{\vct\Sigma}\big\|_F
}
\ \gtrsim
\sqrt{\sum_{k=1}^d
\frac{\lambda_k^\star\left(1-\left(\lambda_k^\star/\|\vct\lambda^\star\|_2\right)^2\right)}
{\mathrm{VLC}_\rho(\sigma_k\mid\mathcal H)}} = \sqrt{\sum_{k=1}^d
\frac{\hts_k(\vct \lambda^\star)}
{\mathrm{VLC}_\rho(\sigma_k\mid\mathcal H)}},
\end{equation*}
where $\gtrsim$ ignores numerical constants and error terms. This yields the desired inequality. Thus the lemma is proved for $\vct{A_\star} = \vct{S}_\star \vct{O}_\star$. Finally, the ratio in the LHS is scale-invariant in $\vct A$: replacing $\vct A$ by $t\vct A$ multiplies both numerator and
denominator by $t$.
\end{proof}

\begin{proof}{\textit{Proof of Lemma~\ref{lem:resolution_tuning}.}}
Recall the parametrization
\[
\vct\lambda_x
=
\vct\lambda^\star\odot\Bigl(\one+\frac{\rho}{\sqrt d}\,\vct A x\Bigr),
\qquad x\in\mathcal X.
\]
By Lemma~\ref{lem:feasibility_app}, it suffices to ensure the two conditions
\[
\langle \vct\lambda^\star,\vct A x\rangle=0
\qquad\text{and}\qquad
\|\vct A x\|_2\le \sqrt d
\qquad\forall x\in\mathcal X,
\]
because they imply $\{\vct\lambda_x\}_{x\in\mathcal X}\subset\Delta_\rho(\vct\lambda^\star)$.

Let ${\vct A}_\star$ be the matrix provided by Lemma~\ref{lem:good_representation_tuning}.
As ensured in the construction, we have ${\vct A}_\star\mathcal X \subset(\vct\lambda^\star)^\perp$, i.e., the first condition is satisfied. Moreover, we have
\begin{equation*}
    \left\|\vct A_\star x\right\|_2 = \left\|\vct S_\star \vct O_\star x\right\|_2
\end{equation*}

Moreover, once again from Lemma \ref{lem:good_representation_tuning}, the matrix $\tilde{A}$ can be rescaled without loss of generality. Thus we choose
\begin{equation*}
    \sup_{x \in \mathcal{X}} \|{\vct{A}}_\star x\|_2 = \sqrt{d}.
\end{equation*}
We now need to determine a sufficient resolution $\rho^* > 0$ under which the proof is valid. The only restriction that remains to be verified is the Inequality \eqref{eq:tune-A-fano}:
\begin{equation*}
    \rho \geq \frac{1}{2}\sqrt{\frac{d \log |\mathcal{X}|}{T\big\|\sqrt{\vct\Gamma}\,\vct A\,\sqrt{\vct\Sigma}\big\|_F^2}}.
\end{equation*}
However, we have $\log |\mathcal{X}| \leq d$. Moreover, by the choice of $\vct A_\star$, we have $\big\|\sqrt{\vct\Gamma}\,\vct A_\star\,\sqrt{\vct\Sigma}\big\|_F^2 = \Theta(d)$. Therefore, it suffices to choose
\begin{equation*}
    \rho \gtrsim \sqrt{\frac{d}{T}}.
\end{equation*}
Hence $\rho^* = \Theta\left(\sqrt{\frac{d}{T}}\right)$, which completes the proof 
\end{proof}

\end{APPENDICES}

\end{document}